\documentclass[12pt,twoside,reqno]{amsart}
\usepackage[colorlinks=true,citecolor=blue]{hyperref}
\usepackage{mathptmx, amsmath, amssymb, amsfonts, amsthm, mathptmx, enumerate, color}
\usepackage{graphicx}
\usepackage{epstopdf}

\PassOptionsToPackage{dvipsnames}{xcolor}
\usepackage{verbatim}
\usepackage{mathtools}
\usepackage{mathrsfs}
\usepackage[vlined]{algorithm2e}
\usepackage{float}
\usepackage[most]{tcolorbox}
\usepackage{xspace}
\usepackage{regexpatch}
\usepackage{colortbl}%
\usepackage{empheq}
\usepackage{bm, bbm}
\usepackage{breakcites}
\mathtoolsset{showonlyrefs}
\usepackage{adjustbox}
\usepackage{tabularx}
\usepackage{subcaption}
\usepackage{enumitem}

\definecolor{thmcolor}{rgb}{0.71,0.14,0.07}
\definecolor{democolor}{rgb}{0.15,0.24,0.55}
\definecolor{lightgray}{rgb}{0.65,0.65,0.65}
\definecolor{hypocolor}{rgb}{0.15,0.44,0.85}

\newtheorem{lemma}{Lemma}[section]
\newtheorem{proposition}{Proposition}[section]
\theoremstyle{definition}
\newtheorem{definition}{Definition}[section]
\newtheorem{example}{Example}[section]
\newtheorem{remark}{Remark}[section]

\newtheorem{assumption}{Assumption}[section]
\numberwithin{equation}{section}

\newtcolorbox{dev}{arc=0pt,colframe=lightgray!65,}

\DeclareMathAlphabet{\mathb}{OML}{cmm}{b}{it}

\def\bs#1{\ensuremath{\boldsymbol{#1}}}          

\newcommand{\N}{\mathbb{N}}		
\newcommand{\R}{\mathbb{R}}		
\newcommand{\p}{\mathrm{p}}		

\newcommand{\syst}[1]{\left \{ \begin{array}{l} #1 \end{array} \right. \kern-\nulldelimiterspace}	

\newcommand{\dom}{\text{\normalfont dom}\,}
\newcommand{\inte}{\text{\normalfont int}\,}

\begin{document}

\setcounter{page}{1}

\vspace*{1.0cm}

\title[Variable Bregman Majorization-Minimization Algorithm]{Variable Bregman Majorization-Minimization Algorithm and its Application to Dirichlet Maximum Likelihood Estimation}
\author[S. Martin, J.-C. Pesquet, G. Steidl, I. Ben Ayed]{Ségolène Martin$^{1,*}$, Jean-Christophe Pesquet$^{2}$, Gabriele Steidl$^{1}$, and Ismail Ben Ayed$^{3}$}
\maketitle
\vspace*{-0.6cm}

\begin{center}
{\footnotesize {\it

$^1$TU Berlin, Germany\\
$^2$Universit\'e Paris-Saclay, Inria, CentraleSup\'elec, CVN, France \\
$^3$\'ETS Montr\'eal, Canada

}}\end{center}

\vskip 4mm {\small\noindent {\bf Abstract.}
We propose a novel Bregman descent algorithm for minimizing a convex function that is expressed as the sum of a differentiable part (defined over an open set) and a possibly nonsmooth term. The approach, referred to as the Variable Bregman Majorization-Minimization (VBMM) algorithm, extends the Bregman Proximal Gradient method by allowing the Bregman function used in the divergence to adaptively vary at each iteration, provided it satisfies a majorizing condition on the objective function.\\
This adaptive framework enables the algorithm to approximate the objective more precisely at each iteration, thereby allowing for accelerated convergence compared to the traditional Bregman Proximal Gradient descent. We establish the convergence of the VBMM algorithm to a minimizer under mild assumptions on the family of metrics used.\\
Furthermore, we introduce a novel application of both the Bregman Proximal Gradient method and the VBMM algorithm to the estimation of the multidimensional parameters of a Dirichlet distribution through the maximization of its log-likelihood. Numerical experiments confirm that the VBMM algorithm outperforms existing approaches in terms of convergence speed.

\noindent {\bf Keywords.}
}

\renewcommand{\thefootnote}{}
\footnotetext{ $^*$Corresponding author.
\par
E-mail address: segolene.tiffany.martin@gmail.com}

\section{Introduction}
The objective of this article is to solve the classical convex minimization problem: \begin{equation}\label{eq:problem} \underset{\bs{x} \in \R^d}{\mathrm{minimize}}\;  F(\bs{x}), \end{equation} where $F = f + g$. Here, the function $f: \R^d \to \R \cup\{ {+\infty} \}$ is convex and differentiable on an open set $\mathcal{D} \subseteq \R^d$, while $g: \R^d \to \R \cup\{ {+\infty} \}$ is proper, convex, and lower semi-continuous, with $\dom g = \{\bs{x}\in \R^d\mid g(\bs{x})< +\infty\} \subseteq \mathcal{D}$.
    
When $\mathcal{D} = \R^d$ and the gradient of $f$ is $L$-Lipschitz continuous, problem \eqref{eq:problem} can be efficiently solved by the \emph{Proximal Gradient} (PG) method \cite{fukushima1981generalized, combettes2005signal}, also known as Forward Backward Splitting. Given an initial point $\bs{x}^{(0)} \in \R^d$, the PG algorithm generates a sequence $(\bs{x}^{(\ell)})_{\ell \in \N}$ by solving \begin{equation}\label{eq:proximal_gradient}
(\forall \ell \in \N)\quad
    \bs{x}^{(\ell+1)} = \arg\min_{\bs{x} \in \R^d} \left\{ g(\bs{x}) + \langle \nabla f(\bs{x}^{(\ell)}), \bs{x} - \bs{x}^{(\ell)} \rangle + \frac{1}{2\gamma_\ell} \|\bs{x} - \bs{x}^{(\ell)}\|^2 \right\},
\end{equation} where $\gamma_\ell < 2/L$.
 

 However, in many practical applications the gradient of $f$ is not Lipschitz continuous, or the domain $\mathcal{D}$ is a strict subset of $\R^d$ (see also recent conditions in \cite{perez2021enhanced}).
 This is for instance the case of minimization problems involving a Kullback-Leibler divergence such as Poisson restoration~\cite{bertero2009image} or estimation of the parameters of a distribution through a Maximum Likelihood approach in statistics \cite{joyce2011kullback}. For such cases, PG is not directly applicable.

 Bregman proximal methods address this limitation. First introduced in~\cite{censor1992proximal} and followed by~\cite{teboulle1992entropic,chen1993convergence,eckstein1993nonlinear}, this family of approaches replace the Euclidean metric in proximal schemes with a Bregman distance $D_h$ with respect to a convex and continuously differentiable function $h: \R^d \to (-\infty, +\infty]$ on its domain. We refer to \cite{benning2021bregman} for an overview of Bregman descent methods. In particular, in the \emph{Bregman Proximal Gradient} (BPG) algorithm~\cite{bauschke2017descent,van2017forward} (a.k.a. Mirror Descent algorithm when the non-differentiable term $g$ is set to zero \cite{ben2001ordered,beck2003mirror}), the $\ell$-th iterates reads
 \begin{equation}\label{eq:bregman_proximal_gradient}
    \bs{x}^{(\ell+1)} = \arg\min_{\bs{x} \in \R^d} \left\{ g(\bs{x}) + \langle \nabla f(\bs{x}^{(\ell)}), \bs{x} - \bs{x}^{(\ell)} \rangle + \frac{1}{L} D_h(\bs{x}, \bs{x}^{(\ell)}) \right\},
\end{equation}
for a fixed constant $L>0$.
This formulation generalizes the PG method in \eqref{eq:proximal_gradient}, which can be recovered by setting $h = \Vert \cdot\Vert^2/2$. 

The BPG method has proven valuable in a range of applications. In imaging, BPG has shown effectiveness in addressing Poisson noise in deblurring problems~\cite{setzer2010deblurring, tan2019inertial, kahl2023remarkable} and has recently been adapted within a Plug-and-Play framework for enhanced performance~\cite{hurault2024convergent}. In parametric statistics, BPG is useful for estimating parameters of distributions, in particular distributions within the exponential family, for which it is closely linked to the Expectation-Maximization algorithm~\cite{nemirovski2009robust, priol2021convergence, kunstner2021homeomorphic}. Moreover, constrained sampling through mirror-Langevin dynamics has also incorporated BPG techniques, illustrating its versatility~\cite{ahn2021efficient}.

 The convergence of BPG has been studied in~\cite{bauschke2017descent,teboulle2018simplified} under the Lipschitz-like condition: \begin{equation}
 Lh - f \text{ is convex on } \inte(\dom h),
 \end{equation}
 where $\inte(S)$ denotes the interior of a subset $S$ of $\R^d$.
 This condition was shown in \cite[Lemma 1]{bauschke2017descent} to be equivalent to the following \emph{majorization} property, for every $\bs{y}\in \inte(\dom h)$, \begin{equation} (\forall \bs{x} \in \inte(\dom h)) \quad f(\bs{x}) \leq f(\bs{y}) + \langle \nabla f(\bs{y}), \bs{x} - \bs{y} \rangle + L D_h(\bs{x}, \bs{y}). \end{equation}
 This condition is also sometimes called an $L$-relative smoothness property of $f$ with respect to
 function $h$.
 Under this property, BPG can be interpreted as a \emph{Majorization-Minimization} (MM) method with the following update rule: \begin{equation}\label{eq:MM_method}
    \bs{x}^{(\ell+1)} = \arg\min_{\bs{x} \in \R^d} q_{\bs{x}^{(\ell)}}(\bs{x}),
\end{equation} where the surrogate function $q$ is defined by \begin{equation}\label{eq:Bregman_surrogate}
    (\forall \bs{x}\in \inte(\dom h)) \quad q_{\bs{y}}(\bs{x})= g(\bs{x}) +f(\bs{y}) + \langle \nabla f(\bs{y}), \bs{x} - \bs{y} \rangle + L D_h(\bs{x}, \bs{y}).
\end{equation}

    In this paper, we introduce an extended version of the BPG framework that allows the Bregman metric $D_h$ in \eqref{eq:bregman_proximal_gradient} to vary at each iteration, adapting to the current iterate. We refer to this adaptive approach as the Variable Bregman Majorization-Minimization (VBMM) algorithm.
    By adapting the Bregman divergence to the local geometry at each iteration, we aim to achieve more accurate approximation to the objective $F$ and to accelerate convergence.

    The Bregman proximal gradient approach with variable Bregman metric has been relatively unexplored. 
    It has been studied in the context of monotone operator theory in \cite{van2017forward} and \cite{bui2021bregman}, where decaying conditions on the Bregman divergences are imposed. Our approach differs by not requiring such conditions. Additionally, \cite{ochs2019non} analyzed the convergence of a generalized majorization-minimization framework, of which VBMM is a special case. Their convergence theory, however, depends on a ``growth condition'' formulated as
    \begin{equation}
        |F(\bs{x}) - q_{\bs{y}}(\bs{x})| \leq \omega(\Vert\bs{x}-\bs{y}\Vert)
    \end{equation}
    where $\omega: [0, +\infty)\longrightarrow [0, +\infty)$ is a differentiable function such that $\omega(0) = 0$, $w'>0$, and $\lim_{t \to +\infty} \omega'(t) = \lim_{t \to +\infty} \omega(t)/\omega'(t)=0$. In practice, however, it may be challenging to identify a suitable function $\omega$ that fulfills these conditions, limiting the practical use of this approach for some applications. Finally, the recent work \cite{rossignol2022bregman} proposed a variable Bregman proximal gradient approach for a Poisson deconvolution problem and showed that using a variable metric can accelerate convergence. However, this work did not provide any convergence analysis.
    
    The paper is organized as follows. In Section \ref{sec:problem_and_algorithm}, we introduce the Variable Bregman Majoriza\-tion-Minimization algorithm. In Section \ref{sec:convergence_analysis}, we provide a convergence analysis that relies on mild assumptions on the Bregman metric, enhancing the practical utility of the method. Finally, Section~\ref{sec:application} is dedicated to an original application of the VBMM algorithm to the Maximum Likelihood estimation of the parameters of a Dirichlet distribution. In particular, we prove that minimization problem admits a solution, derive closed form iterates for the VBMM algorithm, and provide numerical experiments on synthetic data to illustrate the good performance of the algorithm.

    Throughout the article, $\partial \phi$ denotes the Moreau subidfferential of a convex function $\phi\colon \R^d \longrightarrow \R\cup \{+\infty\}$.

\section{Problem and algorithm}\label{sec:problem_and_algorithm}

Let $F\colon \R^d \longrightarrow \R \cup \{+\infty\}$ be defined as
\begin{equation}
    F = f + g,
\end{equation}
where $f\colon \R^d \longrightarrow \R \cup \{+\infty\}$ and $g\colon \R^d \longrightarrow \R \cup \{+\infty\}$ are functions satisfying the following assumption:

\begin{assumption}\label{assumption:f_and_g}~
\begin{enumerate}[label=(\roman*)]
    \item $f$ is convex and 
    differentiable on an open set $\mathcal{D} \subseteq \R^d$,
    \item $g$ is proper, convex, and lower semi-continuous with $\dom g \subseteq \mathcal{D}$,
    \item $F$ is lower bounded.
\end{enumerate}
\end{assumption}
Our objective is to solve the classical convex minimization problem
\begin{equation}\label{eq:main_minimization_pb}
    \underset{\bs{x} \in \R^d}{\mathrm{minimize}}\; F(\bs{x}).
\end{equation}
In particular, if $g$ is the indicator function of a nonempty closed convex
set $\mathcal{C}\subset \mathcal{D}$, the following constrained optimization
problem is obtained:
\begin{equation}\label{eq:main_minimization_pb_constraint}
    \underset{\bs{x} \in \mathcal{C}}{\mathrm{minimize}}\; f(\bs{x}).
\end{equation}

A central mathematical tool in our methodology is the Bregman divergence, originally introduced in \cite{bregman1967relaxation}.

\begin{definition}[Bregman Divergence]
Consider a function $h: \R^d \rightarrow (-\infty, +\infty]$ that is 
differentiable on $\inte(\dom h)$. The Bregman divergence associated with $h$ is defined, for every pair $(\bs{x}, \bs{y})$ in $\dom h \times \inte (\dom h)$, as
\begin{equation}
D_h(\bs{x}, \bs{y}) \coloneqq h(\bs{x}) - h(\bs{y}) - \langle \nabla h(\bs{y}), \bs{x} - \bs{y} \rangle.
\end{equation}
\end{definition}
Essentially, $D_h(\bs{x}, \bs{y})$ quantifies the difference between the value of the function $h$ at a point $\bs{x} \in \dom h$ and its linear approximation around $\bs{x}$, calculated at $\bs{y} \in \inte (\dom h)$.
If $h$ is convex on $\dom h$, then $D_h(\bs{x}, \bs{y})\ge 0$ and $D_h(\cdot, \bs{y})$ is convex. If $h$ is strictly convex on  $\inte ( \dom h)$ and $\bs{x} \in \inte (\dom h)$,
then $D_h(\bs{x}, \bs{y}) = 0$ if and only if $\bs{x} = \bs{y}$. Furthermore,
$D_h(\cdot, \bs{y})$ is strictly convex.
Note that
\begin{equation} \label{starstar}
\nabla_{\bs{x}} D_h(\bs{x},\bs{y}) = \nabla h(\bs{x}) - \nabla h(\bs{y}).
\end{equation}


\begin{definition}[Bregman Majorant Function]\label{de:Bregmaj}
For a given $\bs{y} \in \mathcal{D}$, let $h_{\bs{y}}\colon \R^d \longrightarrow \R \cup \{+\infty\}$ be a lower semi-continuous, 
strictly convex function, which is
differentiable on $\inte (\dom h_{\bs{y}})\supseteq \mathcal{D}$. 
Let
 $q_{\bs{y}}\colon \mathcal{D} \longrightarrow \R$ be defined as 
\begin{equation}\label{eq:bregman_majorant}
q_{\bs{y}}(\bs{x}) \coloneqq f(\bs{y}) + \langle \nabla f(\bs{y}),\bs{x} - \bs{y}\rangle + D_{h_{\bs{y}}}(\bs{x}, \bs{y}).
\end{equation}
We call $q_{\bs{y}}$ a Bregman tangent majorant of $f$ at $\bs{y}$ associated with $h_{\bs{y}}$ if, for every $x \in \mathcal D$,
\begin{equation}\label{eq:bregman_majoran_1t}
       f(\bs{x}) \leq q_{\bs{y}}(\bs{x})
 \end{equation}
 or, in other words,
 $$
 D_f(\bs{x},\bs{y}) \le D_{h_{\bs{y}}}(\bs{x}, \bs{y}).
 $$
\end{definition}

The definition expands upon concepts previously introduced as those in \cite{teboulle2018simplified, bauschke2017descent}. A distinctive aspect of our approach is the flexibility, in the \emph{Majorization-Minimization} (MM) strategy, to allow the Bregman metric to vary according to the reference point $\bs{y}$. 

Let $\left( h_{\bs{y}} \right)_{y \in \mathcal D}$ be a family of function such that, for every $\bs{y} \in \mathcal{D}$, $ h_{\bs{y}} \colon \R^d \longrightarrow \R \cup \{+\infty\}$.  We make the following assumption.
\begin{assumption}\label{assumption:bregman_functions}
For $f$ and $g$ as in Assumption \ref{assumption:f_and_g}, 
the family of functions $\left( h_{\bs{y}} \right)_{y \in \mathcal D}$ fulfills the conditions: for every $\bs{y} \in \mathcal{D}$,
\begin{enumerate}[label=(\roman*)]
    \item\label{assumption:bregman_functionsi} $h_{\bs{y}}$ is a lower semi-continuous
    function which is  strictly convex and
differentiable on $\inte (\dom h_{\bs{y}})$, and $\mathcal{D} \subseteq \inte(\dom h_{\bs{y}})$.
\item\label{assumption:bregman_functionsii} The function
\begin{equation}\label{e:coercivecond}
\bs{x} \mapsto g(\bs{x})
   + \langle \nabla f(\bs{y}),\bs{x}-\bs{y}\rangle +  D_{h_{\bs{y}}}(\bs{x}, \bs{y})
\end{equation}
is coercive.
\item The majorizing property \eqref{eq:bregman_majoran_1t} is satisfied.
\end{enumerate}
\end{assumption}

\begin{remark}
A  sufficient condition for Assumption \ref{assumption:bregman_functions}\ref{assumption:bregman_functionsi}-\ref{assumption:bregman_functionsii} to be satisfied is that 
$h_{\bs{y}}$, $\bs{y} \in \mathcal D$, are lower-semicontinuous  
Legendre functions with $\mathcal{D} \subseteq 
\inte (\dom h_{\bs{y}})$
and the functions 
\begin{equation}
\varphi_{\bs{y}}: \bs{x} \mapsto g(\bs{x}) +
\langle \nabla f(\bs{y}),\bs{x} - \bs{y}\rangle, \quad 
\bs{y} \in \mathcal D
\end{equation}
are lower bounded. This can be seen as follows:
recall that $h_{\bs{y}}$ is a Legendre function if it is 
essentially strictly convex and essentially smooth.
The first property implies that $h_{\bs{y}}$ is strictly convex on $\inte (\dom h_{\bs{y}})\subseteq \dom \partial h_{\bs{y}}$.
As shown in \cite[Thm. 3.7]{bauschke1997legendre},
$D_{h_{\bs{y}}}(\cdot, \bs{y})$ is coercive because $h_{\bs{y}}$ is essentially strictly convex. Since $\varphi_{\bs{y}}$ is lower-bounded, this implies that \eqref{e:coercivecond} is coercive.
\end{remark}

\begin{example}
  A useful example of metrics $(h_{\bs{y}})_{\bs{y} \in \mathcal{D}}$ leading to separable Bregman majorants, and thus easy to minimize, can be defined as follows. For every $\bs{y} \in \mathcal{D}$, let 
    \begin{equation}\label{eq:separable_form_bregman}
    (\forall \bs{x} = (x_i)_{1\le i \le d} \in \R^d)\quad 
       h_{\bs{y}}(\bs{x}) \colon = \sum_{i=1}^d a_i(\bs{y}) \nu_i(x_i) + \sum_{i=1}^d b_i(\bs{y}) \frac{x_i^2}{2},
    \end{equation}
    where, for every  $i \in \{1, \dots, d\}$, the functions $a_i, b_i \colon \R^d \longrightarrow [0, +\infty)$  are continuous, and $\nu_i\colon \R \longrightarrow \R \cup \{+\infty\}$ is convex and 
    differentiable on the interior of its domain  with $\mathcal{D} \subseteq \inte \big(\bigtimes_{j=1}^d \dom {\nu}_j \big)$. Note that the convexity of the differentiable function $\nu_i$ implies that $\nabla \nu_i$ is continuous, see \cite[Corollary 17.42]{bauschke2017convex}.
Functions of the form \eqref{eq:separable_form_bregman} are in particular useful for representing quadratic majorants with a diagonal curvature matrix when setting $a_i = 0$ for all $i \in \{1, \dots, d\}$. This form is also well-suited for the maximum likelihood Dirichlet problem that will be discussed in the application section.
\end{example}

The Variable Bregman Majorization-Minimization (VBMM) algorithm aims at minimizing the sum of the non-smooth function $g$ and 
a Bregman majorant of $f$ calculated from the previous iterate 
$\bs{x}^{(\ell)}$, i.e., 
finding for $\bs{y}= \bs{x}^{(\ell)}$ the minimizer of
\begin{align*}
\Phi_{\bs{y}}(\bs{x}) &\coloneqq
g(\bs{x}) +  q_{\bs{y}}(\bs{x})
\\
&=
g(\bs{x}) +   f(\bs{y}) + \langle \nabla f (\bs{y}), \bs{x}-\bs{y}\rangle +h_{\bs{y}} ( \bs{x}) - h_{\bs{y}}( \bs{y}) 
- \langle \nabla h_{\bs{y}}(\bs{y} ), \bs{x} - \bs{y}\rangle
\\
&=
g(\bs{x}) + h_{\bs{y}}( \bs{x})
+ \langle \nabla f (\bs{y})- \nabla h_{\bs{y}}(\bs{y} ), \bs{x} \rangle
- \langle \nabla f (\bs{y})- \nabla h_{\bs{y}}(\bs{y} ), \bs{y} \rangle
+  f(\bs{y}).
\end{align*}
Neglecting the constants yields Algorithm \ref{algo:VBMM}.
\medskip

\RestyleAlgo{ruled}
\begin{algorithm}[H]
\textbf{Input:} Initial point $\bs{x}^{(0)} \in \mathcal{D}$

\For{$\ell=0, 1, \dots$}{
    $\bs{x}^{(\ell+1)} 
=\arg\min_{\bs{x} \in \R^d} \; 
g(\bs{x}) +
\langle \nabla f(\bs{x}^{(\ell)}),\bs{x} \rangle + D_{h_{\bs{x}^{(\ell)}}}(\bs{x}, \bs{x}^{(\ell)}) $
}
\caption{VBMM Algorithm \label{algo:VBMM}}
\end{algorithm}

\begin{proposition}
Under Assumptions \ref{assumption:f_and_g} and \ref{assumption:bregman_functions},
    Algorithm \ref{algo:VBMM} is well-defined
    and generates a sequence $(\bs{x}^{(\ell)})_{\ell \ge 1}$
    in $\dom g$.
\end{proposition}
\begin{proof}
For  $\bs{x}^{(\ell)} \in \mathcal{D}$, the function  
\begin{equation}
\bs{x}\mapsto 
\langle \nabla f(\bs{x}^{(\ell)}),\bs{x} - \bs{x}^{(\ell)}\rangle + D_{h_{\bs{x}^{(\ell)}}}(\bs{x}, \bs{x}^{(\ell)}) + g(\bs{x})
\end{equation}
is well-defined since $\mathcal{D} \subseteq \inte (\dom h_{\bs{x}^{(\ell)}})$.
This function is proper since $\varnothing \neq \dom g \subseteq \mathcal{D}$.
In addition, the lower-semicontinuity and coercivity properties
in Assumption \ref{assumption:bregman_functions}
implies the existence of a minimizer $\bs{x}^{(\ell+1)}$ of this function,
which necessarily belongs to $\dom g$. This minimizer is unique
since $h_{\bs{x}^{(\ell)}}$ and thus $D_{h_{\bs{x}^{(\ell)}}}(\cdot, \bs{x}^{(\ell)})$ are strictly convex on $\mathcal{D}$.
The result is then shown by induction.
\end{proof}

\begin{remark}
    This strategy of adapting the metric to the current iterate in non-Bregman MM schemes has already been widely studied both theoretically and numerically for the case of quadratic majorants \cite{geman1992constrained,chouzenoux2013majorize,verdun2024fast}. In particular, our Variable Bregman Majorization-Minimization algorithm can be seen as a generalization of the Variable Metric Forward-Backward (VMFB) algorithm \cite{lee2012proximal,combettes2014variable,chouzenoux2014variable} which considers quadratic majorants of the form
    \begin{equation}
        (\forall \bs{x}\in \R^d) \quad q(\bs{x}; \bs{y})=  
        f(\bs{y}) + \langle \nabla f(\bs{y}), \bs{x} - \bs{y} \rangle + 
    \frac{1}{2} \langle \bs{x} - \bs{y}, \bs{A}_{\bs{y}}(\bs{x} - \bs{y}) \rangle,
    \end{equation}
    with $\bs{A}_{\bs{y}}$ a positive definite matrix that depends on the point $\bs{y}$. 
\end{remark}

\begin{remark}
For every $\ell \in \N$,
$\bs{x}^{(\ell+1)}$ is actually the $D_{h_{\bs{x}^{(\ell)}}}$-proximity operator of function $g$
at $\nabla h_{\bs{x}^{(\ell)}}(\bs{x}^{(\ell)})- \nabla f(\bs{x}^{(\ell)})$. 
Algorithm \ref{algo:VBMM} is thus a
variable metric forward-backward splitting method based on a Bregman divergence \cite{bui2021bregman,van2017forward}. The main
difference with the analysis conducted in \cite{bui2021bregman}, in the general context of monotone operator theory, is that we do not
assume that the Bregman divergences $(D_{h_{\bs{x}^{(\ell)}}})_{\ell \in \N}$
satisfy a decaying condition like in \cite[Algorithme 2.4b]{bui2021bregman} (see also \cite[Definition 2.2]{van2017forward}), but we rely
on a Majorization-Minimization property.
When $g=0$, the algorithm simplifies to
\begin{equation}
(\forall \ell \in \N)\quad
\begin{cases}
\widetilde{\bs{z}}^{(\ell)} =
\nabla h_{\bs{x}^{(\ell)}}(\bs{x}^{(\ell)})\\
\bs{z}^{(\ell+1)} =
\widetilde{\bs{z}}^{(\ell)}-\nabla f(\bs{x}^{(\ell)})\\
\bs{x}^{(\ell+1)} = (\nabla h_{\bs{x}^{(\ell)}})^{-1}(\bs{z}^{(\ell+1)}).
\end{cases}
\end{equation}
This shows that the mirror map is modified along the gradient steps based on the current iterate.
\end{remark}

\section{Convergence analysis}\label{sec:convergence_analysis}

Our convergence result is stated in Proposition~\ref{prop:subsequential_conv}. It shows that, if the sequence $(\bs{x}^{(\ell)})_{\ell \in \mathbb{N}}$ generated by Algorithm~\ref{algo:VBMM} lies in a compact subset of $\mathcal{D}$, then subsequential convergence towards a minimizer of $F$ is achieved 
under suitable assumptions on the Bregman functions. 
The existence of such compact set
can be addressed independently. For instance, a sufficient condition is that the objective function is coercive on $\mathcal{D}$. This can be ensured by a proper choice of function $g$.

We start proving that, without further assumptions, Algorithm \ref{algo:VBMM} produces a non increasing sequence of cost values. First, we recall the \emph{Three Points Identity}, which can be found, e.g., in \cite[Lemma 3.1]{chen1993convergence}.

\begin{lemma}[Three Points Identity]\label{lemma:three_points} 
Let $h \colon \R^d \longrightarrow \R \cup \{+\infty\}$ be 
strictly convex and
differentiable on $\inte (\dom h)$. Then, for any three points $\bs{x}, \bs{y}$ in $\inte (\dom h)$ and $\bs{z} \in \dom h$, it holds
\begin{equation}
    D_h(\bs{z}, \bs{x}) - D_h(\bs{z}, \bs{y}) -D_h(\bs{y}, \bs{x}) = \langle \nabla h (\bs{x}) -\nabla h (\bs{y}) , \bs{y} - \bs{z} \rangle.
\end{equation}
\end{lemma}

\begin{lemma}\label{lemma:decrease}
    Suppose that Assumptions \ref{assumption:f_and_g} and \ref{assumption:bregman_functions} are satisfied. Let $(\bs{x}^{(\ell)})_{\ell \in \N}$ be the sequence generated by VBMM Algorithm \ref{algo:VBMM}. Then, for every $\ell \in \N$ and $\bs{v} \in \mathcal{D}$, we have
    \begin{equation}
        F(\bs{x}^{(\ell +1)}) -F(\bs{v}) \leq D_{h_{\bs{x}^{(\ell)}}}(\bs{v}, \bs{x}^{(\ell)}) - D_{h_{\bs{x}^{(\ell)}}}(\bs{v}, \bs{x}^{(\ell +1)})  - D_{f}(\bs{v}, \bs{x}^{(\ell)}).
    \end{equation}
In addition, sequence $(F(\bs{x}^{(\ell)}))_{\ell \in \N}$ is non-increasing.
\end{lemma}

\begin{proof}
    Let $\bs{v} \in \mathcal{D}$ and $\ell \in \N$. 
By definition of $\bs{x}^{(\ell+1)}$ and \eqref{starstar}, there exists $\bs{w}^{(\ell +1)} \in \partial g (\bs{x}^{(\ell +1)})$ such that
\begin{equation}
    \nabla f(\bs{x}^{(\ell)}) = \nabla h_{\bs{x}^{(\ell)}}(\bs{x}^{(\ell)}) - \nabla h_{\bs{x}^{(\ell)}}(\bs{x}^{(\ell+1)}) - \bs{w}^{(\ell +1)}.
\end{equation}
Using the fact that  $\bs{w}^{(\ell +1)}$ is a subgradient of  $g$ at $\bs{x}^{(\ell +1)}$, we obtain
\begin{align}
   \langle \nabla f(\bs{x}^{(\ell)}), \bs{x}^{(\ell +1)} - \bs{v}\rangle &= \langle \nabla h_{\bs{x}^{(\ell)}}(\bs{x}^{(\ell)}) - \nabla h_{\bs{x}^{(\ell)}}(\bs{x}^{(\ell+1)}), \bs{x}^{(\ell +1)} - \bs{v} \rangle - \langle \bs{w}^{(\ell +1)}, \bs{x}^{(\ell +1)} - \bs{v} \rangle \nonumber\\
    &\leq \langle \nabla h_{\bs{x}^{(\ell)}}(\bs{x}^{(\ell)}) - \nabla h_{\bs{x}^{(\ell)}}(\bs{x}^{(\ell+1)}), \bs{x}^{(\ell +1)} - \bs{v} \rangle + g(\bs{v}) - g(\bs{x}^{(\ell+1)}).
\end{align}
Furthermore, it follows from Lemma \ref{lemma:three_points} with $\bs{x}=\bs{x}^{(\ell)} $, $\bs{y} = \bs{x}^{(\ell+1)}$, and $\bs{z} =\bs{v}$ that
\begin{align}
    \langle \nabla f(\bs{x}^{(\ell)}),\bs{x}^{(\ell +1)} - \bs{v}\rangle 
    &\leq  D_{h_{\bs{x}^{(\ell)}}}(\bs{v}, \bs{x}^{(\ell)}) - D_{h_{\bs{x}^{(\ell)}}}(\bs{v}, \bs{x}^{(\ell +1)}) - D_{h_{\bs{x}^{(\ell)}}}(\bs{x}^{(\ell +1)}, \bs{x}^{(\ell)}) \\
    &+ g(\bs{v}) - g(\bs{x}^{(\ell+1)}). \label{eq:identity_1}
\end{align}
On the other hand, the majorization property at iteration $\ell$ reads as
\begin{equation}
    f(\bs{x}^{(\ell+1)}) \leq f(\bs{x}^{(\ell)}) + \langle\nabla f(\bs{x}^{(\ell)}),\bs{x}^{(\ell+1)} - \bs{x}^{(\ell)}\rangle+ D_{h_{\bs{x}^{(\ell)}}}(\bs{x}^{(\ell+1)}, \bs{x}^{(\ell)}),
\end{equation}
and subtracting $f(\bs{v})+\langle\nabla f(\bs{x}^{(\ell)}),\bs{x}^{(\ell+1)} - \bs{v}\rangle $ on both sides, yields
\begin{align}
    f(\bs{x}^{(\ell+1)}) - f(\bs{v}) - \langle\nabla f(\bs{x}^{(\ell)}),\bs{x}^{(\ell+1)} - \bs{v}\rangle  
    &\leq 
    f(\bs{x}^{(\ell)}) - f(\bs{v}) 
    + \langle\nabla f(\bs{x}^{(\ell)}),\bs{v} - \bs{x}^{(\ell)}\rangle \\
    &\quad + D_{h_{\bs{x}^{(\ell)}}}(\bs{x}^{(\ell+1)}, \bs{x}^{(\ell)})\\
    &= - D_{f}(\bs{v}, \bs{x}^{(\ell)}) +  D_{h_{\bs{x}^{(\ell)}}}(\bs{x}^{(\ell+1)}, \bs{x}^{(\ell)}).
    \label{eq:equiv_identity_2}
\end{align}
Finally, summing up \eqref{eq:identity_1} and \eqref{eq:equiv_identity_2}, 
\begin{equation}
        f(\bs{x}^{(\ell +1)}) -f(\bs{v}) + g(\bs{x}^{(\ell +1)}) -g(\bs{v}) \leq D_{h_{\bs{x}^{(\ell)}}}(\bs{v}, \bs{x}^{(\ell)}) - D_{h_{\bs{x}^{(\ell)}}}(\bs{v}, \bs{x}^{(\ell +1)}) - D_{f}(\bs{v}, \bs{x}^{(\ell)}).
    \end{equation}
By setting $\bs{v} = \bs{x}^{(\ell)}$, we deduce that $(F(\bs{x}^{(\ell)}))_{\ell \in \N}$ is non-increasing.
\end{proof}

We shall now prove that, under the condition that the Bregman divergences 
are uniformly lower-bounded by an increasing function of the norm, the difference between two consecutive iterates tends to zero.

\begin{assumption}\label{assumption:lower_bound_euclidean}
   For any nonempty bounded set $\mathcal{C} \subseteq \mathcal{D}$, there exists 
   an increasing function $\rho\colon \R_+ \to \R_+$ vanishing only at 0,
   such that
    \begin{equation}\label{eq:lower_bound_euclidean}
    (\forall (\bs{x}, \bs{y}) \in \mathcal{C}^2)\quad 
      D_{h_{\bs{x}}}(\bs{x}, \bs{y}) 
    \geq 
    \rho(\Vert \bs{x} - \bs{y} \Vert).
    \end{equation}
\end{assumption}

Note that similar conditions are classical in the case when 
$h_{\bs{x}}$ is independent of $\bs{x}$ \cite[Property (*) in Sec. 4.2]{butnariu2003uniform}.
Assumption \ref{assumption:lower_bound_euclidean} is satisfied for Bregman functions of the form given in \eqref{eq:separable_form_bregman}, particularly when, for all $\bs{y}  \in \mathcal{C}$, there exists a constant $c \geq 0$ such that, for every $i \in \{1, \dots, d\}$ $b_i(\bs{y}) \geq c$. In this scenario, \eqref{eq:lower_bound_euclidean} is satisfied with $\rho = c\, (\cdot)^2/2$.

\begin{remark}
Assumption \ref{assumption:lower_bound_euclidean} can be seen as a generalization to the Bregman framework of the assumption made on the metrics in the VMFB algorithm~\cite{chouzenoux2014variable}. Indeed, VMFB was shown to converge to a minimizer under the assumption that there exits $\nu\in (0, +\infty)$ such that
    for all $\bs{y}\in \R^d$, 
    $\nu \bs{I}_d \preceq \bs{A}_{\bs{y}} $, where $\bs{A}_{\bs{y}}$ is the curvature matrix associated with the quadratic majorant at $\bs{y}$.
\end{remark}

\begin{lemma}\label{lemma:difference_to_zero}
    Let $(\bs{x}^{(\ell)})_{\ell \in \N}$ be the sequence generated by the VBMM algorithm \ref{algo:VBMM}.
    Assume 
    that $(\bs{x}^{(\ell)})_{\ell \in \N}$ is bounded.
    Suppose that Assumptions \ref{assumption:f_and_g}, \ref{assumption:bregman_functions}, and \ref{assumption:lower_bound_euclidean} are satisfied. Then $ \bs{x}^{(\ell+1)} - \bs{x}^{(\ell)} \to 0$  as $\ell \to +\infty$.
\end{lemma}

\begin{proof}
    Applying Lemma \ref{lemma:decrease} with $\bs{v}= \bs{x}^{(\ell)}$ yields
    \begin{equation}\label{eq:decrease}
        F(\bs{x}^{(\ell +1)}) -F(\bs{x}^{(\ell)}) \leq  - D_{h_{\bs{x}^{(\ell)}}}(\bs{x}^{(\ell)}, \bs{x}^{(\ell +1)}).
    \end{equation}
According to Assumption \ref{assumption:lower_bound_euclidean}, 
there exists a function $\rho\colon \R_+ \to \R_+$, increasing and vanishing only at 0, such that, for every $\ell \in \N$, 
\begin{align}
    D_{h_{\bs{x}^{(\ell)}}}(\bs{x}^{(\ell)}, \bs{x}^{(\ell +1)}) 
    &\geq \rho(\Vert \bs{x}^{(\ell)} -\bs{x}^{(\ell+1)} \Vert).
\end{align}
Hence, summing from $\ell=0$ to $\ell=L$,
\begin{align}
    \sum_{\ell=0}^L \rho(\Vert \bs{x}^{(\ell)} -\bs{x}^{(\ell+1)} \Vert)
    \leq F(\bs{x}^{(0)}) - F(\bs{x}^{(L+1)})
    \leq F(\bs{x}^{(0)}) - \underline{F},
\end{align}
where $\underline{F}$ is a lower bound on $F$.
This shows that $\sum_{\ell=0}^{+\infty} \rho(\Vert \bs{x}^{(\ell)} -\bs{x}^{(\ell+1)} \Vert)< +\infty$, which implies that
$\rho(\Vert \bs{x}^{(\ell)} -\bs{x}^{(\ell+1)} \Vert) \to 0$,
that is $\bs{x}^{(\ell)} -\bs{x}^{(\ell+1)} \to 0$.
\end{proof}

\begin{assumption}\label{assumption:uniform_continuity}
    For 
    any bounded sequences $(\bs{x}^{(\ell)})_{\ell \in \N}$, $(\bs{y}^{(\ell)})_{\ell \in \N}$ in $\mathcal{D}$ such that $ \bs{x}^{(\ell)} - \bs{y}^{(\ell)}\to 0 $, it holds
    \begin{equation}
    \nabla h_{\bs{y}^{(\ell)}} (\bs{y}^{(\ell)}) -  \nabla h_{\bs{y}^{(\ell)}} (\bs{x}^{(\ell)}) \to 0.
    \end{equation}
\end{assumption}

It is easy to check that  Assumption \ref{assumption:uniform_continuity} is satisfied by Bregman functions of the form \eqref{eq:separable_form_bregman}.

\begin{proposition}\label{prop:subsequential_conv}
Let $(\bs{x}^{(\ell)})_{\ell \in \N}$ be the sequence generated by the VBMM Algorithm \ref{algo:VBMM}.
Assume that there exists a compact subset $\mathcal{C} \subset
\mathcal{D}$, such that $\{\bs{x}^{(\ell)}\}_{\ell \in \N} \subseteq \mathcal{C}$.
%
Further, suppose that Assumptions~\ref{assumption:f_and_g}-\ref{assumption:uniform_continuity} 
are satisfied.
Then the following holds true:
\begin{enumerate}[label=(\roman*)]
\item the set of cluster points of $(\bs{x}^{(\ell)})_{\ell \in \N}$ is a nonempty compact and connected subset of $\operatorname{Argmin} F$;
\item $(F(\bs{x}^{(\ell)}))_{\ell \in \N}$ converges to
the minimum of $F$;
\item  if $F$ admits a unique minimizer, then $(\bs{x}^{(\ell)})_{\ell \in \N}$ converges to this minimizer.
\end{enumerate}
\end{proposition}

\begin{proof}
\begin{enumerate}[label=(\roman*)]
\item 
We show that the cluster points are minimizers of $F$.
Then the fact that the set of cluster points of $(\bs{x}^{(\ell)})_{\ell \in \N}$ is a nonempty compact and connected subset of $\operatorname{Argmin} F$ follows
from the boundness of $(\bs{x}^{(\ell)})_{\ell \in \N}$
and Lemma~\ref{lemma:difference_to_zero}, by invoking
Ostrowski's theorem \cite[Lemma 2.53]{bauschke2017convex}. 

Let $\bar{\bs{x}}\in \mathcal{C}\subset \mathcal{D}$ be a cluster point of $(\bs{x}^{(\ell)})_{\ell \in \N}$. 
There exists a  subsequence $(\bs{x}^{(\ell_j)})_{j \in \N}$ converging to 
$\bar{\bs{x}}$.
Let us show that $\bar{\bs{x}}$ is a critical point of $F$, i.e., $0\in \partial F(\bar{\bs{x}})$, or equivalently, $-\nabla f(\bar x) \in \partial g(\bar x)$. According to the definition of the subdifferential, it remains to show that
    \begin{equation} \label{st}
        (\forall \bs{y}\in \R^d)\quad g(\bs{y}) \geq g(\bar{\bs{x}}) + \langle -\nabla f(\bar{\bs{x}}), \bs{y}- \bar{\bs{x}}\rangle.
    \end{equation}
    By definition of the sequence generated by the algorithm, for all $j \in \N^*$, there exists $\bs{w}^{(\ell_j)} \in \partial g(\bs{x}^{(\ell_j)})$ such that
    \begin{equation}\label{27}
        \bs{w}^{(\ell_j)} = \nabla h_{\bs{x}^{(\ell_j-1)}}(\bs{x}^{(\ell_j-1)}) - \nabla h_{\bs{x}^{(\ell_j-1)}}(\bs{x}^{(\ell_j)}) - \nabla f(\bs{x}^{(\ell_j-1)}).
    \end{equation}
Using again the definition of the subdifferential, 
we get
\begin{equation}
\bs{w}^{(\ell_j)} \in \partial g(\bs{x}^{(\ell_j)}) \quad
\Leftrightarrow \quad
    (\forall \bs{y}\in \R^d)\;\;  g(\bs{y}) \geq g(\bs{x}^{(\ell_j)}) + \langle \bs{w}^{(\ell_j)}, \bs{y}- \bs{x}^{(\ell_j)}\rangle.
\end{equation}
By \eqref{27}, we obtain
\begin{align}\label{eq:decomposition_scalar_product}
    \langle \bs{w}^{(\ell_j)}, \bs{y}- \bs{x}^{(\ell_j)}\rangle  
    &= \langle \nabla h_{\bs{x}^{(\ell_j-1)}}(\bs{x}^{(\ell_j-1)}) - \nabla h_{\bs{x}^{(\ell_j-1)}}(\bs{x}^{(\ell_j)}), \bs{y}- \bs{x}^{(\ell_j)}\rangle\\ 
        & \quad - \langle \nabla f(\bs{x}^{(\ell_j-1)}), \bs{y}- \bs{x}^{(\ell_j)}\rangle.
\end{align}
Since we know from  Lemma~\ref{lemma:difference_to_zero} that
$
\bs{x}^{(\ell_j-1)} - \bs{x}^{(\ell_j)} \to 0,
$
as $j\to +\infty$,
it follows from Assumption~\ref{assumption:uniform_continuity}
that the first term on the right-hand side of \eqref{eq:decomposition_scalar_product} tends to zero as $j\to +\infty$. Additionally, since $f$ is convex and differentiable
on $\mathcal{D}$, its gradient is continuous 
on $\mathcal{D}$ \cite[Corollary 17.42]{bauschke2017convex}.
Consequently, since $\bs{x}^{(\ell_j-1)} 
\to \bar{\bs{x}}$ as $j\to +\infty$, 
\begin{equation}
\langle \nabla f(\bs{x}^{(\ell_j-1)}) , \bs{y}- \bs{x}^{(\ell_j)}\rangle \underset{j \rightarrow +\infty}{\longrightarrow} \langle \nabla f(\bar{\bs{x}}), \bs{y}- \bar{\bs{x}}\rangle.
\end{equation}
Concluding the proof, we employ the lower semi-continuity of $g$ to deduce:
\begin{align}
    (\forall \bs{y}\in \R^d)\quad  g(\bs{y}) &\geq \liminf_{j \rightarrow +\infty} \left\{g(\bs{x}^{(\ell_j)})\right\} + \lim_{j \rightarrow +\infty} \left\{\langle \bs{w}^{(\ell_j)}, \bs{y}- \bs{x}^{(\ell_j)}\rangle\right\}\nonumber \\
    &\geq g(\bar{\bs{x}}) + \langle -\nabla f(\bar{\bs{x}}), \bs{y}- \bar{\bs{x}}\rangle.
\end{align}
Thus, $\bar{\bs{x}}$ is a critical point of $F$, and by convexity of $F$, it is also a minimizer.

\item For all $j \in \N^*$, since $\bs{w}^{(\ell_j)}+\nabla f(\bs{x}^{(\ell_j)})
\in \partial F(\bs{x}^{(\ell_j)})$, we obtain
\begin{equation}
F(\bar{\bs{x}}) \ge F(\bs{x}^{(\ell_j)})+
\langle \bs{w}^{(\ell_j)}+\nabla f(\bs{x}^{(\ell_j)}),
\bar{\bs{x}}-\bs{x}^{(\ell_j)}\rangle,
\end{equation}
that is
\begin{equation}
    0 \le F(\bs{x}^{(\ell_j)})-F(\bar{\bs{x}}) 
    \le \langle \bs{w}^{(\ell_j)},
\bs{x}^{(\ell_j)}-\bar{\bs{x}}\rangle+\langle \nabla f(\bs{x}^{(\ell_j)}),
\bs{x}^{(\ell_j)}- \bar{\bs{x}}\rangle.
\end{equation}
From \eqref{27} and the continuity of $\nabla f$ on $\mathcal{D}$, it follows for $j\to +\infty$ that
\begin{align}
    \langle \bs{w}^{(\ell_j)},
\bs{x}^{(\ell_j)}-\bar{\bs{x}}\rangle
&\to 0\\
\langle \nabla f(\bs{x}^{(\ell_j)}),
\bs{x}^{(\ell_j)}- \bar{\bs{x}}\rangle
&\to 0,
\end{align}
which shows that $\lim_{j \rightarrow +\infty}  F(\bs{x}^{(\ell_j)}) = F(\bar{\bs{x}})$.
Since we have proved in Lemma \ref{lemma:decrease}
that $(F(\bs{x}^{(\ell}))_{\ell \in \N}$ is non increasing, and it is also lower bounded, it converges,  and its limit is $F(\bar{\bs{x}})$.

\item When $F$ has a single minimizer,  $(\bs{x}^{(\ell)})_{\ell \in \N}$ is a bounded sequence with a unique cluster point, and it therefore converges to the unique minimizer.
\end{enumerate}
\end{proof}

\section{Application to the estimation of the parameter of a Dirichlet distribution}\label{sec:application}

In this section, we present a novel application of the VBMM algorithm for estimating the parameters of a Dirichlet distribution. The Dirichlet distribution, defined on the unit simplex in $\R^d$, is widely used in various domains. It has applications in modeling text data, such as word appearance in documents \cite{blei2003latent}, hyperspectral unmixing \cite{nascimento2011hyperspectral}, customer segmentation based on spending patterns \cite{pal2022clustering} and, more recently, in image classification using text-vision models like CLIP \cite{martin2024transductive} and image restoration \cite{MGHS2024}. Additionally, it has been employed for generating DNA sequences \cite{stark2024dirichlet}.

We begin by recalling the definition of the Dirichlet distribution and its associated log-likelihood function. We then establish that the negative log-likelihood is coercive, which guarantees the existence of a minimizer. Following this, we demonstrate how the VBMM algorithm converges to a minimizer of the negative log-likelihood function. By leveraging its variable metric feature, VBMM achieves faster convergence compared to traditional methods.

Finally, through numerical experiments, we illustrate that VBMM outperforms existing methods for solving the Dirichlet Maximum Likelihood Estimation problem, particularly in terms of convergence speed.

\subsection{Maximum likelihood estimation for a Dirichlet distribution}
Let $M \in \N^*$ denote the number of samples and let $(\bs{z}_m)_{1 \leq m \leq M} \in (\Delta_d)^M$, where $\Delta_d$ denotes the open unit simplex of $\R^d$. Assume the vectors $(\bs{z}_m)_{1 \leq m \leq M} $ are samples drawn from a same Dirichlet distribution with parameter $\bs{\alpha}= (\alpha_i)_{1 \leq i \leq d} \in (0, +\infty)^d$, corresponding to the probability density function
\begin{equation}
\label{eq:dirichlet_distribution}
(\forall \bs{z} = (z_i)_{1\le i \le d} \in [0,+\infty)^d)\quad 
 \p \left( \bs{z};\, \bs{\alpha} \right) \coloneqq \frac{1}{\mathcal{B}(\bs{\alpha})} \prod_{i=1}^d z_{i}^{\alpha_i -1} \, \mathbbm{1}_{\bs{z} \in \Delta_d},
\end{equation}
where 
$$\mathcal{B}(\bs{\alpha}) \coloneqq  \prod_{i=1}^d \Gamma(\alpha_i) \Big/\Gamma\left(\sum_{i=1}^d \alpha_i\right)$$
and $\Gamma$ denotes the Gamma function. 

To determine the parameter of this distribution, we intend to  minimize the negative log-likelihood function, defined 
as
\begin{equation}\label{eq:def_negative_log_likelihood_dirichlet}
(\forall \bs{\alpha} = (\alpha_i)_{1\le i \le d} \in \R^d)\quad 
  f(\bs{\alpha}) \coloneqq 
  \begin{cases}
  \displaystyle M G(\bs{\alpha})-\sum_{m = 1}^M  \sum_{i=1}^d (\alpha_i -1) \ln z_{m,i} & \mbox{if $\bs{\alpha} \in \mathcal{D}\coloneqq (0, +\infty)^d$,}\\
  +\infty & \mbox{otherwise,}
  \end{cases}
\end{equation}
where
\begin{equation}
(\forall \bs{\alpha} = (\alpha_i)_{1\le i \le d} \in \mathcal{D})\quad 
G(\bs{\alpha}) \coloneqq  \sum_{i=1}^d \ln \Gamma(\alpha_i) - \ln \Gamma\Big(\sum_{i=1}^d \alpha_i\Big).
\end{equation}

\subsection{Existence of a unique minimizer}

\begin{proposition}
\label{prop:coercivity}
    Assume that $d\geq 2$ and $M \geq 2$. If at least two samples in the set $\{\bs{z}_m\}_{1 \leq m \leq M}$ are not identical, then function $f$ is lower-semicontinuous, strictly convex, and coercive.   
    Thus, it admits a unique minimizer.
\end{proposition}

\begin{proof}~\\
\emph{\textbf{Lower-semicontinuity.}}
Function $f$ is continuous on its domain $\mathcal{D}$. Let us study what happens at the boundary of this domain. Let $(\bs{\alpha}^{(\ell)})_{\ell \in \N}$ be a sequence of $(0, +\infty)^d$ converging to
$\bs{\alpha}^*$, where $\bs{\alpha}^* $ belongs to the boundary of $(0, +\infty)^d $. If $\left\{i \in \{1, \dots, d\}\mid \, \alpha_i^* = 0 \right\} \neq \{1, \dots, d\}$, then it is clear that $f(\bs{\alpha}^{(\ell)}) \underset{\ell \rightarrow +\infty}{\longrightarrow} + \infty$.
Now let $\bs{\alpha}^*=\bs{0}$. 
Then, since $\Gamma(\alpha) \sim \frac{1}{\alpha}$ as $ \alpha \to 0+$, we get
\begin{equation}
    \frac{\Gamma(\alpha_1) \dots \Gamma(\alpha_d)}{\Gamma( \sum_{i=1}^d \alpha_i)} \sim \frac{1}{\alpha_1 \ldots \alpha_d} \sum_{i=1}^d \alpha_i = \sum_{i=1}^d \frac{1}{\prod_{j \neq i} \alpha_j }
\end{equation}
and, consequently,
\begin{align}
    G(\bs{\alpha})
     &\sim \ln\left( \sum_{i=1}^d \frac{1}{\prod_{j \neq i} \alpha_j^{(\ell)} }\right) \underset{\ell \rightarrow +\infty}{\longrightarrow} + \infty,
\end{align}
so that $\lim_{\ell \rightarrow +\infty} f(\bs{\alpha}^{(\ell)}) = +\infty= f(\bs{\alpha}^*)$. This shows that $f$ is lower-semicontinuous.


\noindent\emph{\textbf{Coercivity.}}
Let us show that
\begin{equation} \label{e:coerc}
\lim_{\substack{\Vert \bs{\alpha}\Vert\rightarrow +\infty\\ \bs{\alpha}\in \mathcal{D}}} f(\bs{\alpha}) =+\infty.
\end{equation}
According to a variant of Stirling's formula \cite[page 257, 6.1.38]{abramowitz1948handbook}, for all $\alpha >0$, there exists $\theta \in (0, 1)$ such that 
\begin{align}\label{eq:stirling}
    \Gamma(\alpha) &= \frac{\Gamma(\alpha+1)}{\alpha} 
    = \alpha^{-1}\sqrt{2\pi} \alpha^{\alpha + 1/2} \exp\Big(-\alpha + \frac{\theta}{12 \alpha} \Big) \nonumber\\
    &= \sqrt{2\pi} \alpha^{\alpha - 1/2} \exp\Big(-\alpha + \frac{\theta}{12 \alpha} \Big).
\end{align}
For all vectors $\bs{\alpha} = (\alpha_i)_{1 \leq i \leq d} \in (0, +\infty)^d$, we define $\bar{\alpha} \coloneqq \sum_{i=1}^d \alpha_i$. Additionally, we denote by $(\theta_i)_{1 \leq i \leq d} \in (0, 1)^d$ the set of constants corresponding to each $\alpha_i$ as per equation \eqref{eq:stirling}. Similarly, $\bar{\theta}\in (0,1)$ is associated with $\bar{\alpha}$. 
We can write, for all $\bs{\alpha}\in (0, +\infty)^d$, 
\begin{align}
    G(\bs{\alpha}) &= \sum_{i=1}^d \left( (\alpha_i - 1/2)\ln\alpha_i - \alpha_i + \frac{\theta_i}{12\alpha_i}+ \frac{1}{2}\ln(2 \pi) \right) \nonumber \\
    & \quad - \left(\bar{\alpha} -1/2 \right) \ln \bar{\alpha}  + \bar{\alpha} -  \frac{\bar{\theta}}{12 \bar{\alpha}}- \frac{1}{2}\ln(2 \pi) ,\nonumber \\
    &= A(\bs{\alpha}) + B(\bs{\alpha}),
\end{align}
where 
\begin{align}
    A(\bs{\alpha}) &\coloneqq\sum_{i=1}^d (\alpha_i - \frac12) \ln\alpha_i - (\bar{\alpha} - 1/2) \ln\bar{\alpha},\\
    B(\bs{\alpha}) &\coloneqq \frac{1}{12} \left[\sum_{i=1}^d \frac{\theta_i}{\alpha_i} - \frac{\bar{\theta}}{\bar{\alpha}}\right] + \frac{d-1}{2}\ln(2\pi).
\end{align}
Setting $\eta_i := \frac{\alpha_i}{\bar{\alpha}} \in (0, 1]$, we get
\begin{align}
    A(\bs{\alpha}) &= \sum_{i=1}^d (\eta_i\bar{\alpha} -1/2) \ln(\eta_i\bar{\alpha}) - (\bar{\alpha} - 1/2) \ln\bar{\alpha},\nonumber \\
    &= \sum_{i=1}^d \eta_i\bar{\alpha} \ln\eta_i + \sum_{i=1}^d \eta_i\bar{\alpha} \ln\bar{\alpha} - \frac{1}{2}\sum_{i=1}^d \ln\eta_i - \frac{d}{2} \ln\bar{\alpha} - (\bar{\alpha} - 1/2) \ln\bar{\alpha}, \nonumber \\
    &= \bar{\alpha} \sum_{i=1}^d \eta_i \ln\eta_i - \frac{1}{2}\sum_{i=1}^d \ln\eta_i + \frac{1-d}{2} \ln\bar{\alpha}.
\end{align}
Since, for all $i\in \{1, \dots, d\}$, $\eta_i \leq 1$, the following lower bound is obtained: 
\begin{equation}
    A(\bs{\alpha}) \geq \bar{\alpha} \sum_{i=1}^d \eta_i \ln\eta_i + \frac{1-d}{2} \ln\bar{\alpha}.
\end{equation}
Additionally, since for all $i \in \{1, \dots, d\}$, $\theta_i > 0 $ and $\bar{\theta} < 1$, we obtain
\begin{equation}
    B(\bs{\alpha}) > -\frac{1}{12 \bar{\alpha}} + \frac{d-1}{2} \ln(2 \pi).
\end{equation}
Therefore, we can lower-bound $G$ as follows
\begin{equation}
    G(\bs{\alpha)} > \bar{\alpha} \sum_{i=1}^d \eta_i \ln\eta_i + \frac{1-d}{2} \ln\bar{\alpha} - \frac{1}{12 \bar{\alpha}} + \frac{d-1}{2} \ln(2 \pi).
\end{equation}
Let us now go back to the objective function \eqref{eq:def_negative_log_likelihood_dirichlet}. The following holds:
\begin{align}
    f(\bs{\alpha}) 
    &= M G(\bs{\alpha})- \sum_{i=1}^d (\alpha_i -1) \ln \left( \prod_{m=1}^M z_{m,i} \right)\nonumber \\
    &=M \left( G(\bs{\alpha})- \sum_{i=1}^d (\eta_i \bar{\alpha} -1) \ln\tilde{z}_i \right)
\end{align}
where 
\begin{equation}
   (\forall i \in \{1, \dots, d\}) \quad \tilde{z}_i \coloneqq \left( \prod_{m=1}^M z_{m,i} \right)^{1/M}.
\end{equation}
Then,
\begin{align}
     f(\bs{\alpha}) &> M \left( - \sum_{i=1}^d (\eta_i \bar{\alpha} -1) \ln\tilde{z}_i + 
    \bar{\alpha} \sum_{i=1}^d \eta_i \ln\eta_i + \frac{1-d}{2} \ln\bar{\alpha} - \frac{1}{12 \bar{\alpha}} + \frac{d-1}{2} \ln(2 \pi)
    \right), \nonumber \\
    &= M \left( \bar{\alpha} \sum_{i=1}^d \eta_i  \ln\left(\frac{\eta_i}{ \tilde{z}_i}\right) + \sum_{i=1}^d \ln\tilde{z}_i
   + \frac{1-d}{2} \ln\bar{\alpha} - \frac{1}{12 \bar{\alpha}} + \frac{d-1}{2} \ln(2 \pi)
    \right).
\end{align}
Let us recall the definition of the Kullback-Leibler (KL), divergence:
\begin{multline}
   (\forall \bs{u} = (u_i)_{1\le i \le d} \in [0,+\infty)^d)
   (\forall  \bs{v} = (v_i)_{1\le i \le d} \in (0, +\infty)^d) \\ D_{\mathrm{KL}} (\bs{u},\bs{v}) = \sum_{i=1}^d u_i \ln\left( \frac{u_i}{v_i}\right) - \sum_{i=1}^d u_i + \sum_{i=1}^d v_i.
\end{multline}
Since $\sum_{i=1}^d \eta_i = 1$, the following holds, for 
$\bs{\eta} = (\eta_i)_{1\le i \le d}$
and $\tilde{\bs{z}} = (\tilde{z}_i)_{1\le i \le d}$,
\begin{equation}
    D_{\mathrm{KL}} (\bs{\eta},\tilde{\bs{z}}) = \sum_{i=1}^d \eta_i \ln\left( \frac{\eta_i}{\tilde{z}_i}\right) - 1 + \sum_{i=1}^d \tilde{z}_i,
\end{equation}
which implies that
\begin{align}
     f(\bs{\alpha}) &> M \left( \bar{\alpha} \Big(D_{\mathrm{KL}} (\bs{\eta},\tilde{\bs{z}}) +1 - \sum_{i=1}^d \tilde{z}_i \Big) + \sum_{i=1}^d \ln \tilde{z}_i
   + \frac{1-d}{2} \ln\bar{\alpha} -\frac{1}{12 \bar{\alpha}} + \frac{d-1}{2} \ln(2 \pi)
    \right).
\end{align}
Using the nonnegativity of the KL divergence, 
we deduce that
\begin{align}\label{eq:lower_bound_f}
     f(\bs{\alpha}) &> M \left( \bar{\alpha} \Big(1 - \sum_{i=1}^d \tilde{z}_i \Big) + \sum_{i=1}^d \ln\tilde{z}_i
   + \frac{1-d}{2} \ln\bar{\alpha}- \frac{1}{12 \bar{\alpha}} + \frac{d-1}{2} \ln(2 \pi)
    \right).
\end{align}
If we show that $1 - \sum_{i=1}^d \tilde{z}_i > 0$, it follows that the term on the right-hand side of equation \eqref{eq:lower_bound_f} tends to infinity as $\Vert \bs{\alpha}\Vert$ goes to infinity. 
According to the inequality of arithmetic and geometric means, for all $i \in \{1, \dots, d\}$,
\begin{equation}
    \tilde{z}_i = \left(\prod_{m=1}^M z_{m,i}  \right)^{1/M} \leq \frac{1}{M} \sum_{m = 1}^M z_{m,i},
\end{equation}
with equality if and only if for all $i \in \{1, \dots, d\}$, $z_{m,i} = z_{0, i}$.
Since the samples $(\bs{z}_m)_{1 \leq m \leq M}$ belong to the unit simplex of $\R^d$,
\begin{align}\label{eq:inequality_AG}
    \sum_{i=1}^d \tilde{z}_i &\leq \frac{1}{M} \sum_{i=1}^d \sum_{m = 1}^M z_{m,i} \leq 1.
\end{align}
The inequality \eqref{eq:inequality_AG} is strict if there exists $i\in \{1, \dots, d\}$ and $(m, \ell)\in \{1, \dots, M\} $ with $m\neq \ell$ such that $z_{m,i} \neq z_{\ell, i}$. In other words, the inequality is strict if at least two samples are not identical, which has been assumed. Finally, we have proved that the coercivity condition 
\eqref{e:coerc} holds.

\noindent \emph{\textbf{Strict convexity.}}
The Dirichlet distribution is part of the Exponential family. 
In addition, it is associated to the sufficient statistics
\[
(\forall z = (z_i)_{1\le i \le m} \in \Delta_d)\quad 
\bs{T}(\bs{z}) = (-\ln z_i)_{1\le i \le M}.
\]
Since there is no linear dependence between the components of $\bs{T}(\bs{z})$, the Dirichlet distribution defines a minimal exponential family. In such a case,
the negative log-likelihood is known to be strictly convex \cite[Thm. 1.13]{brown1986fundamentals}, \cite{barndorff2014information}. As an alternative proof,
it is shown in \cite{ronning1989maximum}, that the Hessian of the negative log-likelihood is positive definite.
\end{proof}



\subsection{Existence of Bregman majorants}

Building upon the following lemma, we construct a Bregman tangent majorant of $f$.

\begin{lemma}[\cite{erdogan2002monotonic}]\label{lemma:specific_quad_maj}
    Let $\varphi$ be a twice-continuously differentiable function on $[0,+\infty)$.
    Assume that $\varphi''$ is decreasing on $[0,+\infty)$. Let $t\in [0,+\infty)$
    and let 
    \begin{equation}\label{eq:defcourb}
        c(t) = \begin{cases}
        \varphi''(0) & \text{if } t = 0,\\
        \displaystyle 2\frac{\varphi(0)-\varphi(t)+\varphi'(t)t}{t^2} & \text{otherwise.}
        \end{cases}
    \end{equation}
    Then, for every $x\in [0,+\infty)$, it holds
    \begin{equation}
    \varphi(x) \le \varphi(t) + \varphi'(t)(x-t)
    + \frac12 c(t) (x-t)^2.
    \end{equation}
\end{lemma}

\begin{proposition}\label{lemma:bregman_majorant_dirichlet}
Let $\bs{\beta} \in (0, +\infty)^d$. The negative log-likelihood in \eqref{eq:def_negative_log_likelihood_dirichlet} admits a Bregman tangent majorant at $\bs{\beta}$ associated with the Bregman function
\begin{equation}\label{eq:def_bregman_function_dirichlet}
 (\forall \bs{\alpha} = (\alpha_i)_{1\le i \le d} \in (0, +\infty)^d)\quad  h_{\bs{\beta}}(\bs{\alpha}) =  h^{{\rm cst}} (\bs{\alpha}) +h_{\bs{\beta}}^{{\rm var}} (\bs{\alpha})
\end{equation}
where $h^{{\rm cst}}$ and $h_{\bs{\beta}}^{{\rm var}}$ are respectively the constant and variable components defined as
\begin{equation}
\begin{cases}
 \displaystyle h^{{\rm cst}}(\bs{\alpha})  \coloneqq -  M\sum_{i=1}^d\ln\alpha_i\\
 \displaystyle h_{\bs{\beta}}^{{\rm var}}(\bs{\alpha})\coloneqq M \sum_{i=1}^d c(\beta_{i})\frac{\alpha_i^2}{2} ,
 \end{cases}
\end{equation}
where $c\colon [0, +\infty) \longrightarrow [0, +\infty)$ is defined by Lemma \ref{lemma:specific_quad_maj} applied to the function $\varphi=\ln \Gamma(\cdot +1)$.
\end{proposition}

\begin{proof}  
Recalling that $\ln \Gamma (\cdot+1) = \ln \Gamma + \ln$, we decompose $f$ into
\begin{equation}
    f(\bs{\alpha}) = f_1(\bs{\alpha}) + f_2(\bs{\alpha}) + f_3(\bs{\alpha}),
\end{equation}
where 
\begin{align}
    f_1(\bs{\alpha}) &= -\sum_{m = 1}^M  \left(\sum_{i=1}^d (\alpha_i -1) \ln z_{m,i} + \ln \Gamma\left(\sum_{i=1}^d\alpha_i\right) \right),\\
  f_2(\bs{\alpha}) &=   M\sum_{i=1}^d \ln \Gamma(\alpha_i+1), 
\end{align}
and 
\begin{equation}
    f_3(\bs{\alpha}) = -M\sum_{i=1}^d \ln\alpha_i.
\end{equation}
We now derive upper tangent bounds at $\bs{\beta}$ for each of these functions. Recalling that the Gamma function is log-convex, $f_1$ is the sum of a linear function and a concave function, and the following tangent inequality holds
\begin{equation}\label{eq:majorant_f1}
    f_1(\bs{\alpha}) \leq f_1(\bs{\beta}) + \nabla f_1(\bs{\beta})^\top (\bs{\alpha}-\bs{\beta}).
\end{equation}
In addition, the function $\ln \Gamma (\cdot +1)$ is twice-continuously differentiable on $(0, +\infty)^d$ and its second derivative, which is a shifted version of the Trigamma function, admits the following expansion
\begin{equation}
(\forall t \in \R)\quad
(\ln \Gamma (\cdot +1))''(t) = 
\sum_{\ell=0}^{+\infty } \frac{1}{(t + 1 +\ell)^2},
\end{equation}
provided that $-t\not\in \mathbb{N}$.
This shows that $(\ln \Gamma (\cdot +1))''$ is decreasing on $[0,+\infty)$. Therefore, Lemma~\ref{lemma:specific_quad_maj} applies to $\ln \Gamma (\cdot +1)$, yielding the upper-bound
\begin{equation}
    f_2(\bs{\alpha}) \leq f_2(\bs{\beta}) + \nabla f_2(\bs{\beta})^\top (\bs{\alpha}-\bs{\beta}) + \frac{M}{2 }\sum_{i=1}^d c(\beta_i) (\alpha_i-\beta_i)^2,
\end{equation}
which can also be written as
\begin{equation}\label{eq:majorant_f2}
    f_2(\bs{\alpha}) \leq f_2(\bs{\beta}) + \nabla f_2(\bs{\beta})^\top (\bs{\alpha}-\bs{\beta}) + D_{h_{\bs{\beta}}^{\text{var}}}(\bs{\alpha}, \bs{\beta}).
\end{equation}
Finally, by definition of the Bregman divergence and since $f_3$ coincides with $h^{\text{cst}}$, we obtain
\begin{equation}\label{eq:majorant_f3}
    f_3(\bs{\alpha}) = f_3(\bs{\beta}) + \nabla f_3(\bs{\beta})^\top (\bs{\alpha}-\bs{\beta}) + D_{h^{\text{cst}}}(\bs{\alpha}, \bs{\beta}).
\end{equation}
Thus, since $ D_{h^{\text{cst}} + h_{\bs{\beta}}^{\text{var}}}= D_{h^{\text{cst}}}+ D_{h_{\bs{\beta}}^{\text{var}}}$, we deduce from \eqref{eq:majorant_f1}, \eqref{eq:majorant_f2}, and \eqref{eq:majorant_f3} that
\begin{equation}
    f(\bs{\alpha}) \leq f(\bs{\beta}) + \nabla f(\bs{\beta})^\top (\bs{\alpha}-\bs{\beta}) + D_{h_{\bs{\beta}}}(\bs{\alpha}, \bs{\beta}).
\end{equation}
\end{proof}

\begin{remark}\label{rem:separable_bregman_form_dirichlet}
    The majorants established in Proposition~\ref{lemma:bregman_majorant_dirichlet} can be written in the form \eqref{eq:separable_form_bregman}, with for all $i \in \{1, \dots, d\}$, for all $\bs{\beta}\in (0, +\infty)^d$, $a_i(\bs{\beta})=M$ and $b_i(\bs{\beta})=M c(\beta_i)$, and for all $t\in (0, +\infty)$, $\nu_i(t) = -\ln t$.
Additionally, the curvature function $c$ defined in \eqref{eq:defcourb} is represented in Figure \ref{fig:curvature_dirichlet}. It satisfies $c(0)= \pi^2/6$ and
\begin{equation}
   (\forall t \in [0, +\infty))\quad c(t) >0 ~\text{and}~\lim_{t \rightarrow +\infty} c(t)=0.
\end{equation}
The family of Bregman functions defined in equation \eqref{eq:def_bregman_function_dirichlet}  satisfies the assumptions of Proposition~\ref{prop:subsequential_conv} for suitable choices of function $g$. For all $ \boldsymbol{\beta} \in (0, +\infty)^d $, the function $ h_{\boldsymbol{\beta}} $ is strictly convex and differentiable on its domain which is 
$\mathcal{D} = (0, +\infty)^d$. Moreover, the majorization property is verified as per Proposition~\ref{lemma:bregman_majorant_dirichlet}, thereby satisfying Assumption~\ref{assumption:bregman_functions}. Regarding Assumptions~\ref{assumption:lower_bound_euclidean} and \ref{assumption:uniform_continuity}, they are clearly met as indicated in the remarks made after stating these assumptions.
\end{remark}

\begin{figure}[htb]
    \centering
    \includegraphics{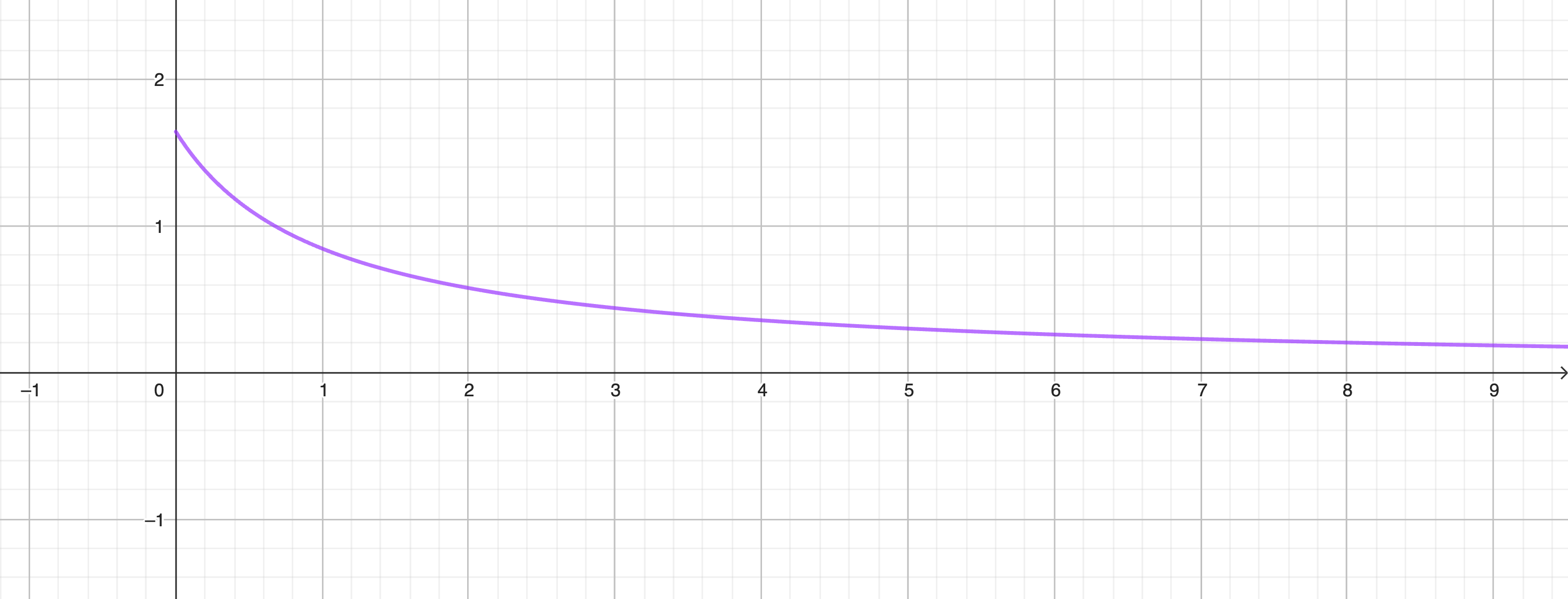}
    \caption{Curvature function \eqref{eq:defcourb} when $\varphi=\ln\Gamma(\cdot+1)$.}
    \label{fig:curvature_dirichlet}
\end{figure}

\subsection{Algorithm}

For a given value of $\bs{\beta}=(\beta_i)_{1\le i \le d}\in (0, +\infty)^d$, we can find a closed-form expression of the minimizer $\widehat{\bs{\alpha}}$ of the Bregman majorant given by Proposition  \ref{lemma:bregman_majorant_dirichlet}.
Since the majorant has a separable form, each component
$\widehat{\alpha}_i$, $i\in \{1,\ldots,d\}$,
of $\widehat{\bs{\alpha}}$ is the minimizer of the function defined on $(0,+\infty)$ as
\[
\alpha_i \mapsto \left((\ln \Gamma)'(\beta_i)
-(\ln \Gamma)'\Big(\sum_{j=1}^d\beta_j\Big)
-\frac{1}{M}\sum_{m=1}^M \ln z_{m,i}\right)
(\alpha_i-\beta_i)+
\frac{c(\beta_i)}{2}  (\alpha_i-\beta_i)^2
- \ln \alpha_i+\frac{\alpha_i-\beta_i}{\beta_i},
\]
where, as already mentioned, the curvature $c(\beta_i)$ is positive.
We deduce that $\widehat{\alpha}_{i}$ is the unique positive root of the second order polynomial equation
\begin{equation}
c(\beta_{i}) \alpha_{i}^2
+ \delta_{i}(\boldsymbol{\beta})
\alpha_{i} = 1,
 \end{equation}
where $\delta_{i}(\boldsymbol{\beta})$ is given by
\begin{equation}
\delta_{i}(\bs{\beta})  := (\ln \Gamma)'(\beta_{i}+1)
- (\ln \Gamma)'\Big(\sum_{j=1}^d \beta_{j}\Big)
-c(\beta_{i}) \beta_{i}
- \frac{1}{M}\sum_{m=1}^M \ln z_{m,i}.
\end{equation}
Therefore, we obtain the following closed-form expression for the minimizer:
\begin{equation}\label{eq:update_MM_alpha}
\widehat{\alpha}_{i}
= \frac{-\delta_{i}(\bs{\beta})+
\sqrt{\big(\delta_{i}(\bs{\beta})\big)^2
+ 4 c(\beta_{i})}}
{2c(\beta_{i})}.
\end{equation}
The final Variable Bregman MM updates are described in Algorithm \ref{algo:VBMM_dirichlet}. Note that it only requires the evaluation of the log-Gamma function and of its derivative,
the Digamma function.

\begin{center}
\RestyleAlgo{ruled}
	\begin{algorithm} 
Initialize $\boldsymbol{\alpha}^{(0)} \in (0, +\infty)^d$.\\
\For{$\ell = 0, 1, \ldots,$}
{\For{$i \in \{1, \dots, d\}$}{
$\displaystyle
c_i^{(\ell)} = 2\left(-\ln \Gamma(\alpha_{i}^{(\ell)} +1)+\alpha_{j}^{(\ell)}(\ln\Gamma)'(\alpha_{i}^{(\ell)}+1)\right)/(\alpha_{i}^{(\ell)})^2$\\
$\displaystyle
\delta_{i}^{(\ell)} = 
(\ln \Gamma)'(\alpha_{i}^{(\ell)}+1) - (\ln\Gamma)'\left(\sum_{j=1}^d \alpha_{j}^{(\ell)} \right)
-c_i^{(\ell)} \alpha_{i}^{(\ell)}
-\frac{1}{M}\sum_{m=1}^M  \ln z_{m,i}.$\\
$\displaystyle
\alpha_{i}^{(\ell+1)}
= \left(-\delta_{i}^{(\ell)}+
\sqrt{(\delta_{i}^{(\ell)})^2
+ 4 c_i^{(\ell)}} \right)\Big/
2c_i^{(\ell)}.$
}}
\caption{VBMM for Dirichlet parameter estimation\label{algo:VBMM_dirichlet}}
\end{algorithm}
\end{center}

\subsection{Numerical experiments}

All numerical experiments were run in Python 3.8 on an Apple M1 CPU.

\subsubsection{Unconstrained case}

We first consider the case when $f$ is defined by\eqref{eq:def_negative_log_likelihood_dirichlet} and the penalty function $g$
is zero. The cost function thus reduces to $F=f$.
We apply the VBMM algorithm  (see Algorithm \ref{algo:VBMM_dirichlet}) to minimize $F$. 

We compare the computational efficiency of VBMM with two state-of-the-art algorithms used to solve the Dirichlet maximum-likelihood problem. 
The first one is the algorithm described in \cite{ronning1989maximum,narayanan1991algorithm,minka2000estimating}, which is a Newton algorithm adapted to deal with the constraint $\bs{\alpha} \in (0, +\infty)^d$ and requires computing the second derivative of the log-Gamma function at each iteration. The second one, proposed in \cite{minka2000estimating}, constructs an upper bound on the negative log-likelihood, which is tight at the current iterate. Unlike ours, this upper bound has no closed-form minimizer. However, the sub-optimization problem can be rewritten as a fixed-point problem, which can be solved with a Newton method separately on each component. Finally, we highlight the advantages of using Bregman functions that adapt dynamically to the iterates by comparing the Variable Bregman Majorization-Minimization (VBMM) algorithm with the fixed-metric Bregman Majorization-Minimization (BMM) approach. In BMM, the variable component $\bs{\alpha} \mapsto h_{\bs{\beta}}^{\text{var}}(\bs{\alpha})$ is replaced with a function independent of $\bs{\beta}$, namely
\begin{equation}
    \bs{\alpha} \mapsto M \sup_{t \in [0, +\infty)} \left\{ c(t) \right\} \sum_{i=1}^d  \frac{\alpha_i^2}{2} =  \frac{M\pi^2}{12}
     \sum_{i=1}^d  \alpha_i^2.
\end{equation}

\paragraph{Experimental setup}

We conduct a series of experiments where we set the dimension $d$of the data to $1000$ and the number of samples $M$ to $500$ across all experiments. We sample from a Dirichlet distribution with parameter $\bs{\alpha}_{\text{true}} \in \R^d$ for different values of $\bs{\alpha}_{\text{true}}$, allowing us to scan a wide range of means and variances for the data distribution. We first define the following three vectors: 
\begin{itemize}
    \item $\bs{m}_1 = \bs{1}_N$;
    \item $\bs{m}_2 = (m_{2,i})_{1 \leq i \leq d}$ where, for all $i \in \{1, \dots, d\}$, $m_{2,i}= 10$ if $i = 1$, $m_{2,i}=1$ otherwise;
    \item $\bs{m}_3 = (m_{3,i})_{1 \leq i \leq d}$ where, for all $i \in \{1, \dots, d\}$, $m_{3,i}= i$.
\end{itemize}
For $j \in \{1, 2, 3\}$, we define the normalized vector $\bar{\bs{m}}_j = \bs{m}_j / \sum_{i=1}^d m_{j, i} \in \Delta_d$. 
We also introduce three scaling factors $s_1 = 100$, $s_2 = 10$, and $s_3 = 1$. Given $\bar{\bs{m}} \in \{\bar{\bs{m}}_1, \bar{\bs{m}}_2, \bar{\bs{m}}_3\}$ and $s \in \{s_1, s_2, s_3\}$, we set the true parameter $\bs{\alpha}_{\text{true}}$ as
\begin{equation}
    \bs{\alpha}_{\text{true}} = s \bar{\bs{m}}.
\end{equation}
In this case, the mean of the Dirichlet distribution is $\bar{\bs{m}}$, and the variance is $\frac{\bar{\bs{m}}(1-\bar{\bs{m}})}{s+1}$. Therefore, a large value of $s$ results in a small variance of the samples.

Once $\bs{\alpha}_{\text{true}}$ is set, we draw our data samples from the Dirichlet distribution with the chosen true parameter. Subsequently, we aim to estimate the Dirichlet parameter solution to the Maximum Likelihood problem. For each method, the initial estimate $\alpha^{(0)}$ is uniformly set to a vector with value $10\, \mathbf{1}_N$. Each experimental setup was averaged over $1000$ random experiments.

\paragraph{Results}

To assess the convergence speed of all three methods on the experimental setup previously described, we evaluate the relative squared error (RSE) with respect to the optimal parameter $\bs{\alpha}_{\text{opt}}$ at each iteration $\ell$ against time in seconds. The RSE at iteration $\ell$ is given by
$
    \frac{\|\bs{\alpha}^{(\ell)} - \bs{\alpha}_{\text{opt}}\|^2}{\|\bs{\alpha}_{\text{opt}}\|^2}
$
where $\bs{\alpha}^{(\ell)}$ is the estimated parameter at iteration $\ell$.
Note that, as the number of samples is finite, 
the maximizer of the likelihood
$\bs{\alpha}_{\text{opt}}$ differs from $\bs{\alpha}_\text{true}$ in general.

In Figure \ref{fig:comparison_dim_1000}, we present convergence plots displaying the distance to the optimum versus time for different values of $\bs{\alpha}_\text{true}$. We observe that, in every configuration, VBMM converges faster than the three other methods.

\begin{figure}[H]
\captionsetup{skip=0pt, justification=raggedright}%
\tabcolsep=2pt
\begin{tabularx}{\textwidth}{cXXX}
& \multicolumn{1}{c}{$s_1$} &\multicolumn{1}{c}{$s_2$} & \multicolumn{1}{c}{$s_3$} \\
\raisebox{2.25cm}{\rotatebox[origin=center]{90}{$\bs{m}_1$}} &
 \includegraphics[width=0.33\textwidth]{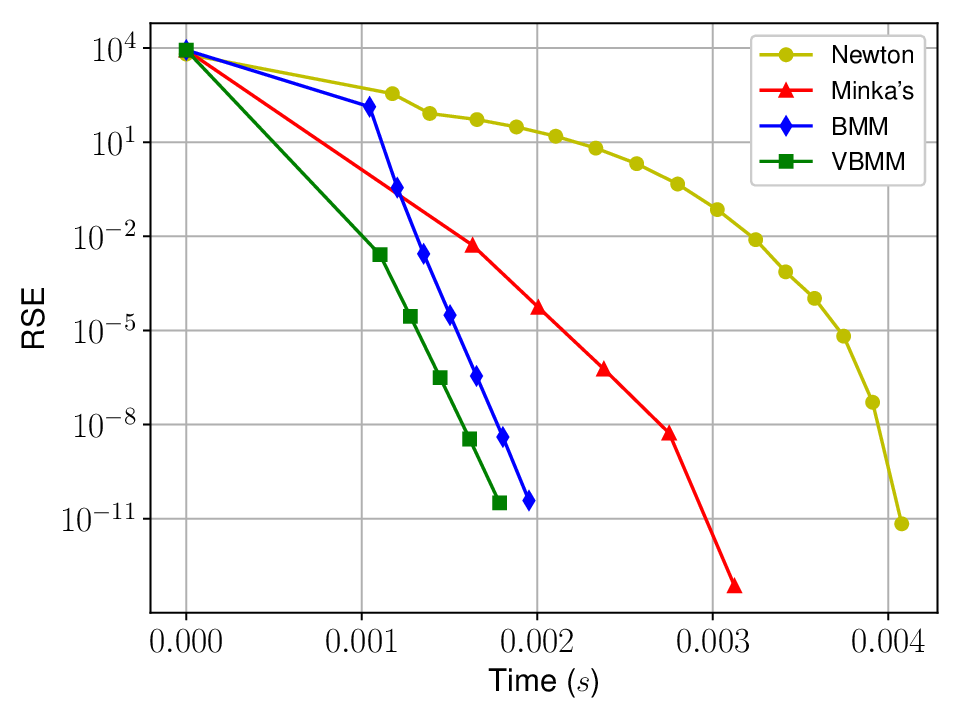} & \includegraphics[width=0.33\textwidth]{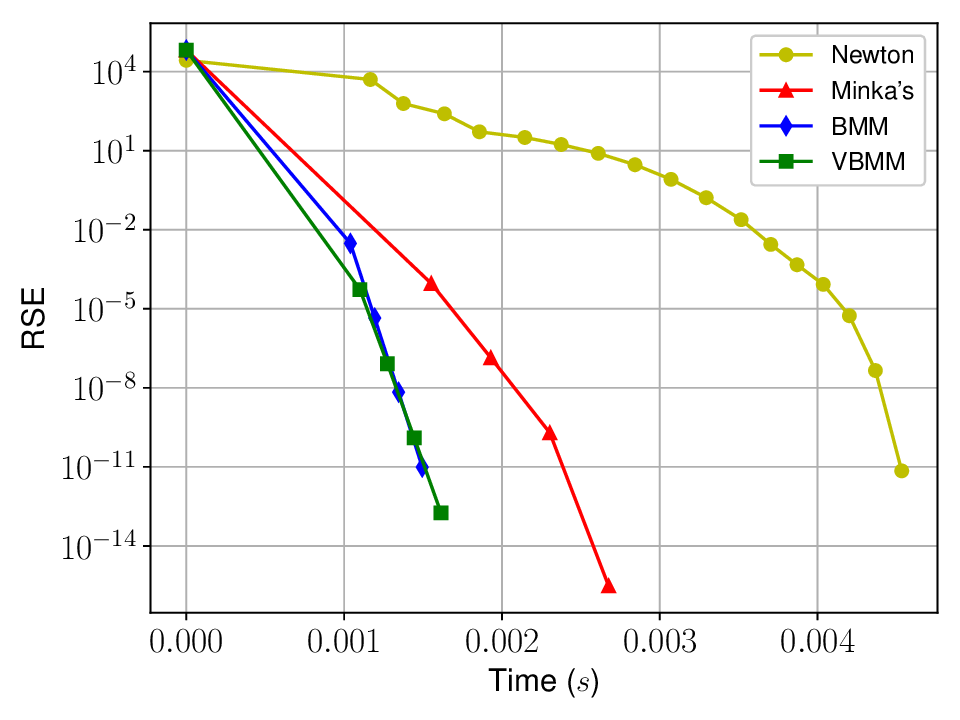} &
 \includegraphics[width=0.33\textwidth]{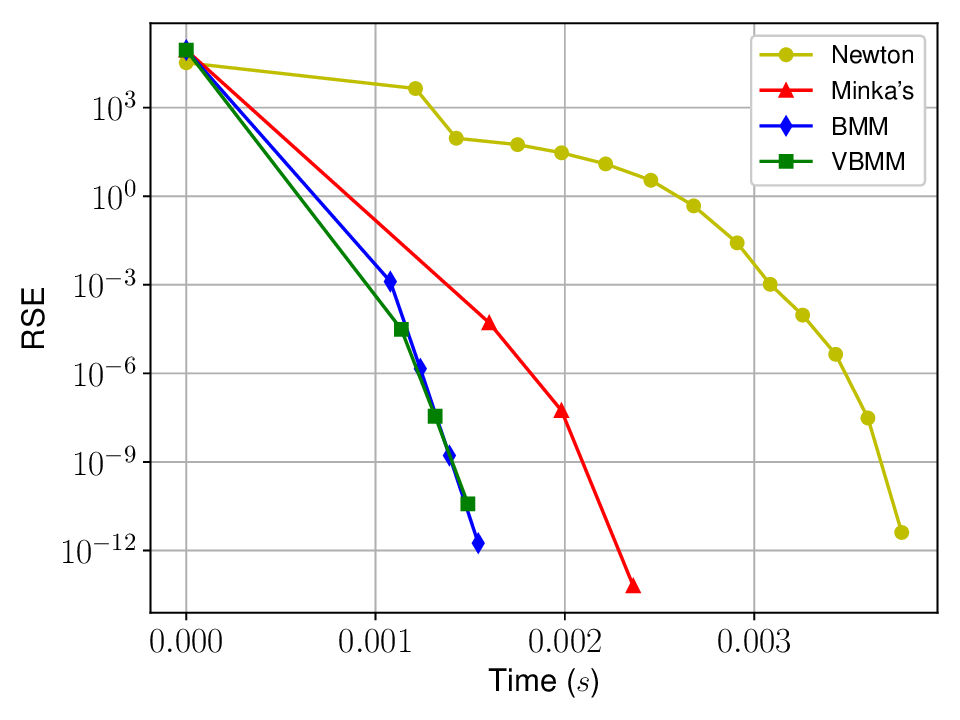}\\
 \raisebox{2.5cm}{\rotatebox[origin=center]{90}{$\bs{m}_2$}} &
 \includegraphics[width=0.33\textwidth]{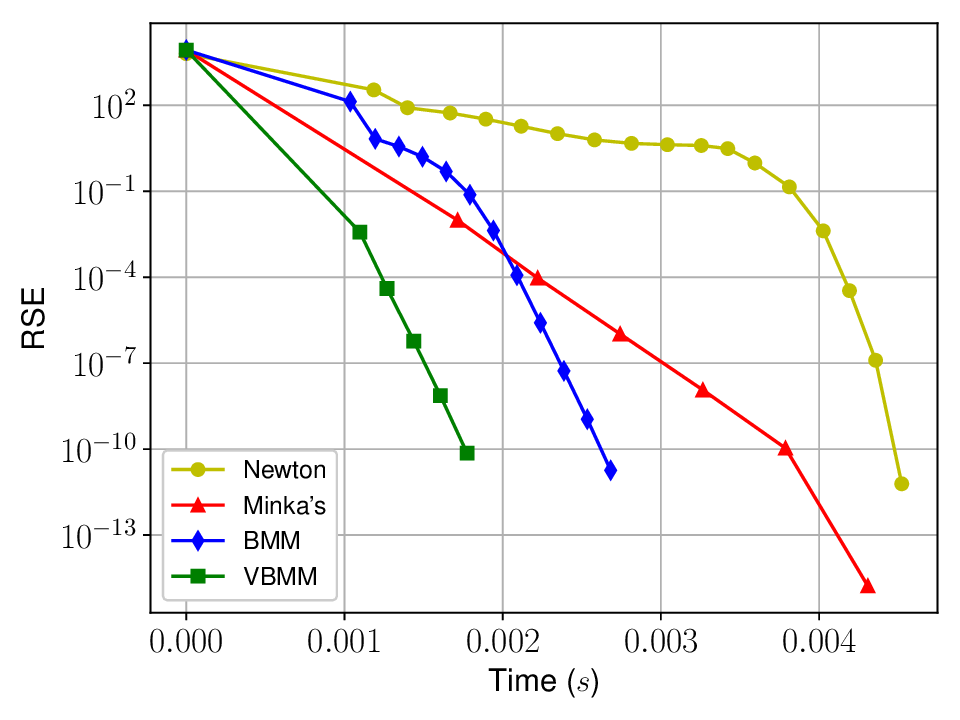} & \includegraphics[width=0.33\textwidth]{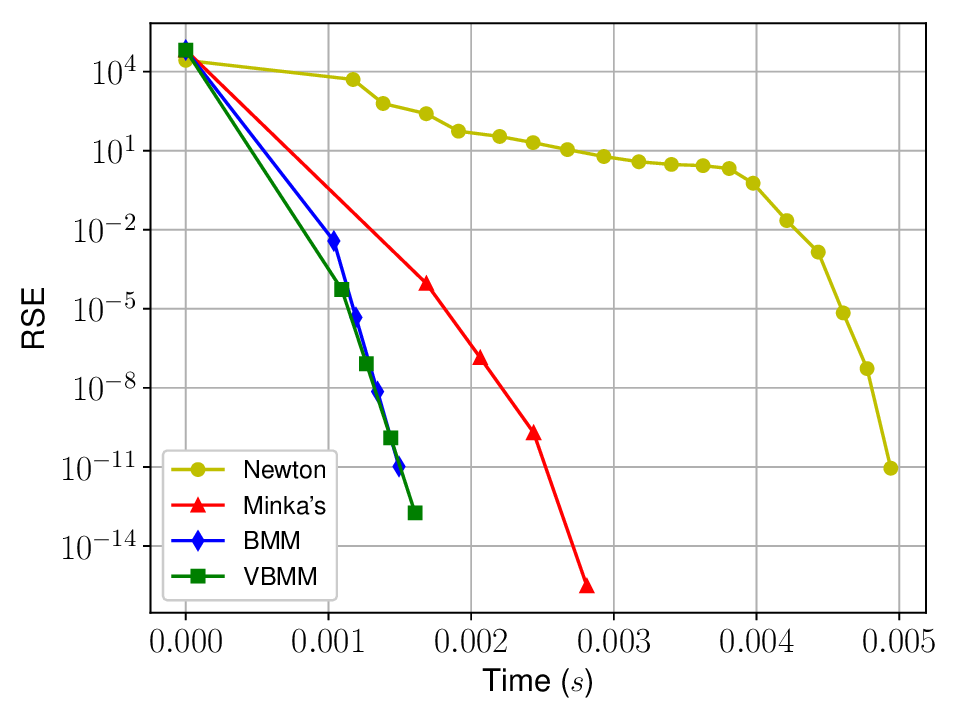} &
 \includegraphics[width=0.33\textwidth]{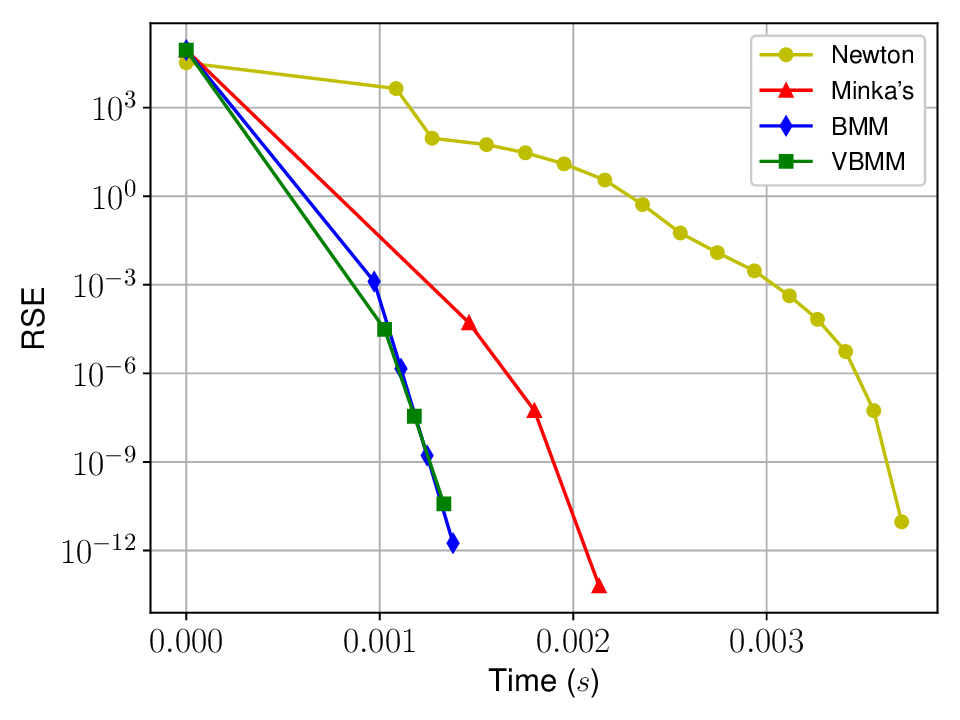}\\
 \raisebox{2.25cm}{\rotatebox[origin=center]{90}{$\bs{m}_3$}} &
 \includegraphics[width=0.33\textwidth]{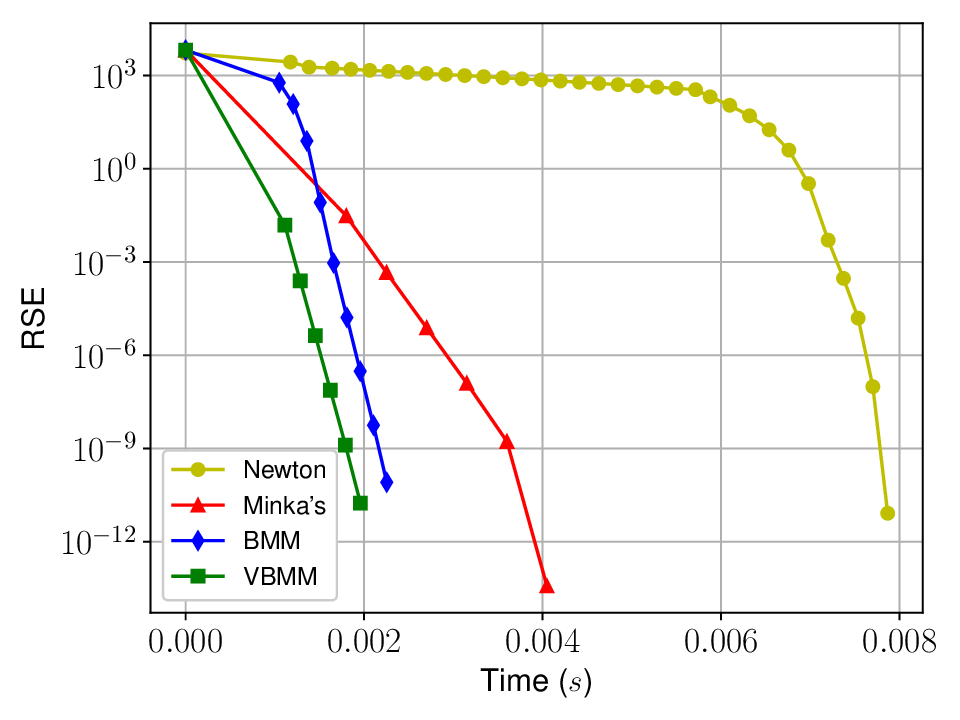} & \includegraphics[width=0.33\textwidth]{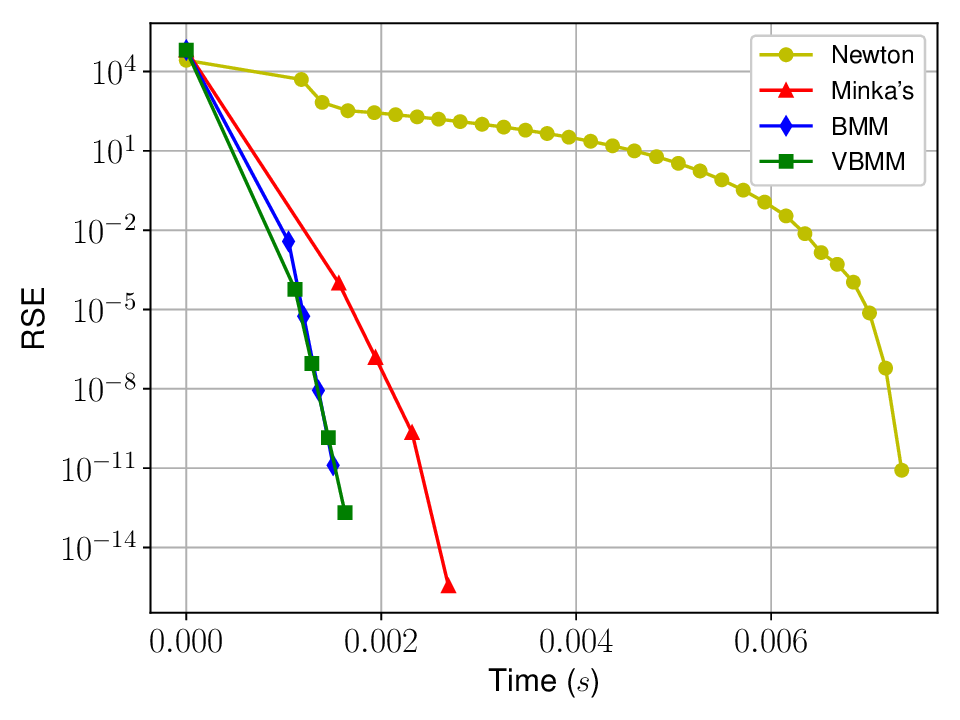} &
 \includegraphics[width=0.33\textwidth]{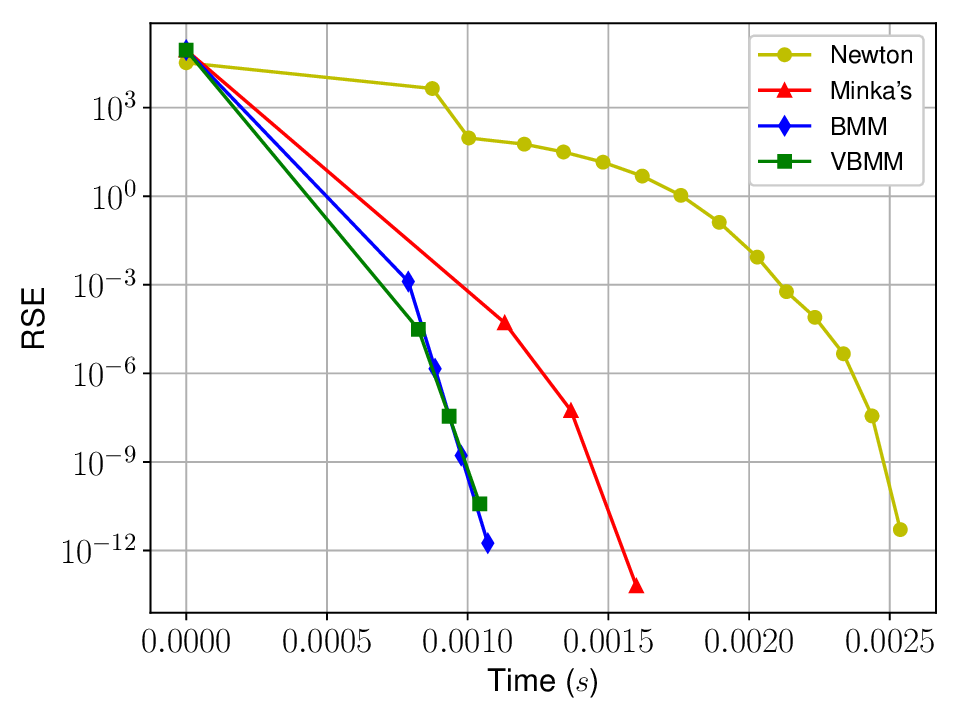}
\end{tabularx}
\caption{Distance to optimum versus time for different values of $\bs{\alpha}_\text{true}$, with $M=500$ samples and sample size $d=1000$. Rows from top to bottom: $\bs{\alpha}_\text{true}$ defined with respectively $\bs{m}_1$, $\bs{m}_2$, and $\bs{m}_3$. Columns from left to right: $\bs{\alpha}_\text{true}$ defined with respectively $s_1$, $s_2$, and $s_3$.}
\label{fig:comparison_dim_1000}
\end{figure}

\subsubsection{Use of a separable constraint}

Our approach also allows for the inclusion of a non-zero function $g$, which can be useful for enforcing constraints on the estimated Dirichlet parameters. Specifically, we consider a function having the following separable form:
\[
(\forall \bs{x} = (x_i)_{1\le i \le d}\in \R^d)\quad 
g(\bs{x}) := \sum_{i=1}^d \iota_{[r_i^{-}, r_i^{+}]}(x_i),
\]
where for each $i \in \{1, \dots, d\}$, $(r_i^{-}, r_i^{+}) \in (0, +\infty)^2$. 

An example of such a constraint arises frequently in the Latent Dirichlet Allocation (LDA) problem, where setting $r_i^{-}=\epsilon$ with $\epsilon>0$ a constant arbitrarily close to zero, and $r_i^{+} = 1$ is common. In practice, this constraint reflects the sparsity of word counts in documents.

In this case, minimizing the Bregman majorant simply reduces to minimizing a convex one-variable function on $[r_i^{-}, r_i^{+}]$, which is equivalent to projecting the unconstrained update onto $[r_i^{-}, r_i^{+}]$. The complete algorithm is detailed in Algorithm \ref{algo:VBMM_dirichlet_constraint}. In our experiments, we set $r_i^{-}=10^{-10}$ and $r_i^{+} = 1$, for all $i\in \{1, \dots, d\}$.

In Figure \ref{fig:VBMM_with_constraint}, we plot both the RSE and the negative log-likelihood evaluated at the iterates of VBMM as a function of time. As theoretically expected, the loss is monotonically decreasing (linearly) and converges fast. 
Note that neither of the two alternative methods we compared against previously are designed to handle constrained problems, highlighting an additional benefit from our approach.

\begin{center}
\RestyleAlgo{ruled}
	\begin{algorithm} 
Initialize $\boldsymbol{\alpha}^{(0)} \in (0, +\infty)^d$.\\
\For{$\ell = 0, 1, \ldots,$}
{\For{$i \in \{1, \dots, d\}$}{
$\displaystyle
c_i^{(\ell)} = 2\left(\ln \Gamma(1)-\ln \Gamma(\alpha_{i}^{(\ell)} +1)+(\ln\Gamma)'(\alpha_{i}^{(\ell)}+1)\alpha_{i}^{(\ell)}\right)/(\alpha_{i}^{(\ell)})^2$\\
$\displaystyle
\delta_{i}^{(\ell)} = 
(\ln \Gamma)'(\alpha_{i}^{(\ell)}+1) - (\ln\Gamma)'\left(\sum_{j=1}^d \alpha_{j}^{(\ell)} \right)
-c_i^{(\ell)} \alpha_{i}^{(\ell)}
- \frac{1}{M}\sum_{m=1}^M  \ln z_{m,i}$\\
$\displaystyle
\alpha_{i}^{(\ell+1)}
= \mathrm{proj}_{[r_i^{-}, r_i^{+}]}\; \left((-\delta_{i}^{(\ell)}+
\sqrt{(\delta_{i}^{(\ell)})^2
+ 4 c_i^{(\ell)}} )\Big/
2c_i^{(\ell)}\right).$
}}
\caption{VBMM for Dirichlet parameter estimation with a box constraint\label{algo:VBMM_dirichlet_constraint}}
\end{algorithm}
\end{center}

\begin{figure}[H]
    \centering
    \includegraphics[width=0.43\linewidth]{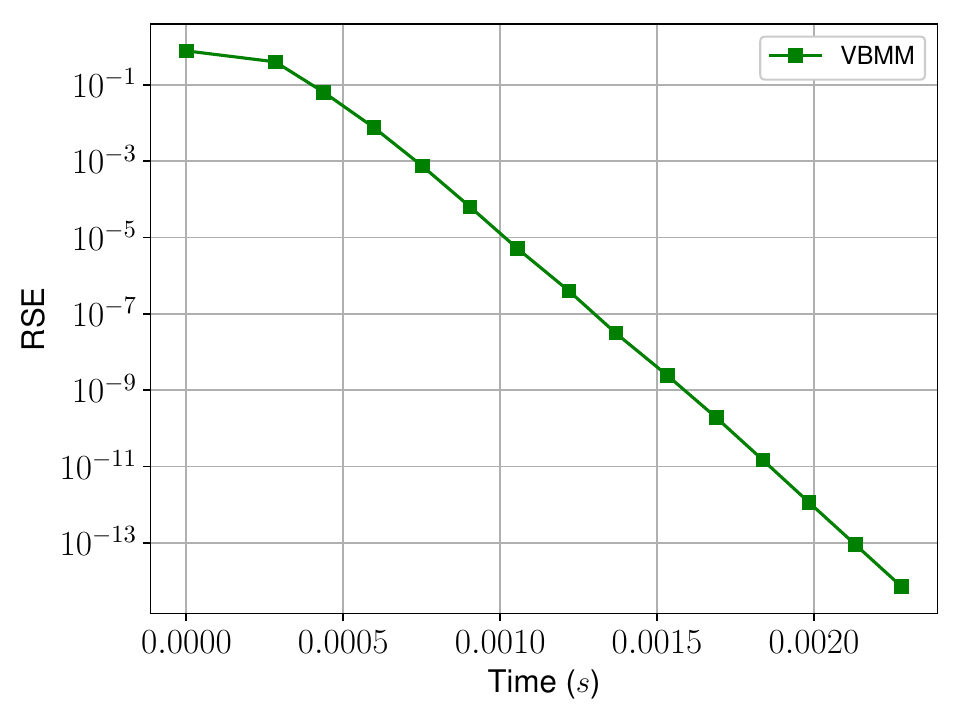}
    \includegraphics[width=0.43\linewidth]{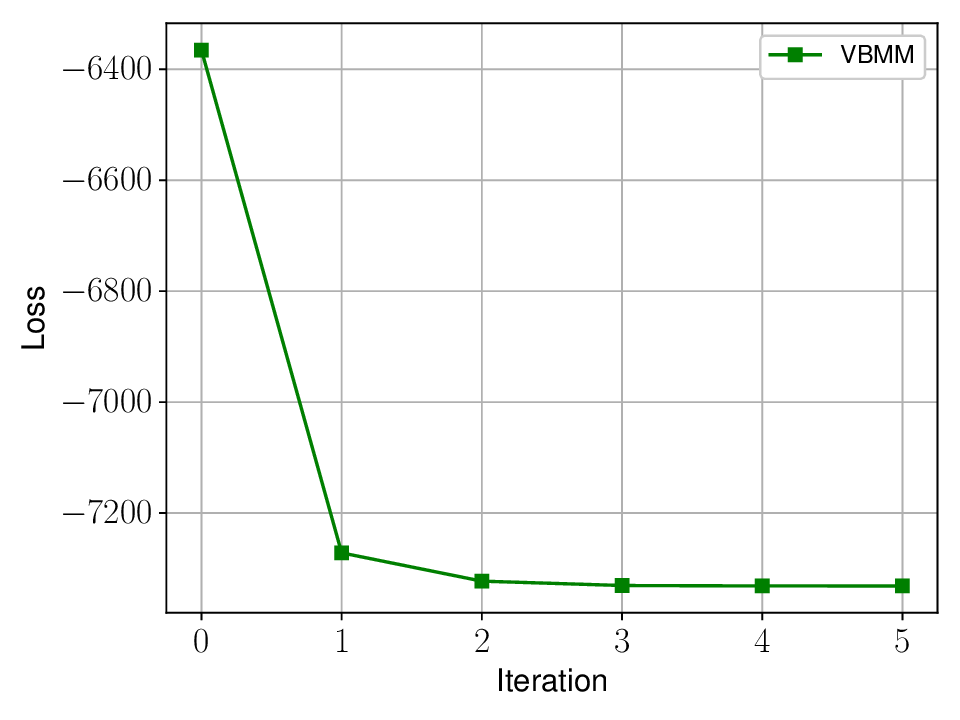}
    \caption{(Left) RSE versus lapsed time. (Right) Function $f$ minus the optimal value $f^*$ versus elapsed time. The loss is averaged over 1000 experiments where $\bs{\alpha}_{\text{true}}$ is uniformly sampled in $[0, 2]^{1000}$. The number of samples for each experiment is $M=500$.
    }
    \label{fig:VBMM_with_constraint}
\end{figure}

\section{Conclusion}

We have introduced the Variable Bregman Majorization-Minimization (VBMM) algorithm as a versatile extension of the Bregman Proximal Gradient method. By allowing the Bregman divergence to adapt dynamically at each iteration, VBMM increases flexibility in optimization and achieves faster convergence. We established the subsequential convergence of the algorithm under mild assumptions on the family of Bregman metrics. A novel application to the estimation of Dirichlet distribution parameters demonstrated the practical efficiency of VBMM, outperforming existing methods in terms of convergence speed.

Future research could focus on strengthening our convergence results. For instance, leveraging the Kurdyka-\L{}ojasiewicz property, as in \cite{teboulle2018simplified}, might lead to stronger convergence results. Moreover, relaxing the convexity requirement on the objective function $f$ could broaden the applicability of VBMM to a wider class of optimization problems.

\vskip 6mm
\noindent{\bf Acknowledgments}

\noindent  The authors thank Michael Adipoetra for his help on the numerical experiments of this paper.
S\'egol\`ene Martin was supported by the DFG funded Cluster of Excellence EXC 2046 MATH+ (project ID: AA5-8).

\bibliographystyle{ieeetr}

\bibliography{biblio}  

\end{document}